\documentclass[11pt]{article}
\usepackage[dvips]{graphicx}
\usepackage{amssymb,amsmath,color}
\usepackage{url}

\usepackage[mathcal]{eucal}

\DeclareMathAlphabet\mathbfcal{OMS}{cmsy}{b}{n}

\usepackage[colorlinks=true, allcolors=blue, pagebackref]{hyperref}       
\renewcommand*{\backrefalt}[4]{%
    \ifcase #1 \footnotesize{(not cited)}%
    \or        \footnotesize{(cited on page~#2)}%
    \else      \footnotesize{(cited on pages~#2)}%
    \fi}

\usepackage[numbers]{natbib}
\newcommand{\BEAS}{\begin{eqnarray*}}
\newcommand{\EEAS}{\end{eqnarray*}}
\newcommand{\BEA}{\begin{eqnarray}}
\newcommand{\EEA}{\end{eqnarray}}
\newcommand{\BEQ}{\begin{equation}}
\newcommand{\EEQ}{\end{equation}}
\newcommand{\BIT}{\begin{itemize}}
\newcommand{\EIT}{\end{itemize}}
\newcommand{\BNUM}{\begin{enumerate}}
\newcommand{\ENUM}{\end{enumerate}}
\newcommand{\BA}{\begin{array}}
\newcommand{\EA}{\end{array}}

\newcommand{\var}{\mathop{ \rm var}}

\newcommand{\tr}{\mathop{ \rm tr}}

\newcommand{\idm}{I}
\newcommand{\rb}{\mathbb{R}}

\newcommand{\BlackBox}{\rule{1.5ex}{1.5ex}}  
\newcommand{\ds}{\displaystyle }

\newenvironment{proof}{\par\noindent{\bf Proof\ }}{\hfill\BlackBox\\[2mm]}

\newtheorem{proposition}{Proposition}

\newcommand{\mysec}[1]{Section~\ref{sec:#1}}
\newcommand{\eq}[1]{Eq.~(\ref{eq:#1})}
\newcommand{\myfig}[1]{Figure~\ref{fig:#1}}

 \def \hphi{ \widehat{\varphi} }
\def \hS{ \widehat{ \Sigma} }
\def \S{  { \Sigma} }

\def \E{{\mathbb E}}
\def \P{{\mathbb P}}

\def \P{{\mathbb P}}

\title{High-dimensional analysis of double descent \\ for linear regression with random projections}

\author{Francis Bach\\
Inria,  Ecole Normale Sup\'erieure \\
PSL Research University \\
{\small \url{francis.bach@inria.fr}}}

\date{\today}

\begin{document}
\maketitle

\begin{abstract}
We consider linear regression problems with a varying number of random projections, where we provably exhibit a double descent curve for a fixed prediction problem, with a high-dimensional analysis based on random matrix theory.
We first consider the ridge regression estimator and review earlier results using classical notions from non-parametric statistics, namely degrees of freedom, also known as effective dimensionality.  
We then compute asymptotic equivalents of the generalization performance (in terms of squared bias and variance) of the minimum norm least-squares fit with random projections, providing simple expressions for the double descent phenomenon. 
\end{abstract}

\section{Introduction}

Over-parameterized models estimated with some form of gradient descent come in various forms, such as linear regression with potentially non-linear features, neural networks, or kernel methods. The double descent phenomenon can be seen empirically in several of these models~\cite{belkin2019reconciling,geiger2019scaling}: Given a fixed prediction problem, when the number of parameters of the model is increasing from zero to the number of observations, the generalization performance traditionally goes down and then up, due to overfitting. Once the number of parameters exceeds the number of observations, the generalization error decreases again, as illustrated in \myfig{doubledescent}.

\begin{figure}
\begin{center}
\includegraphics[scale=.34]{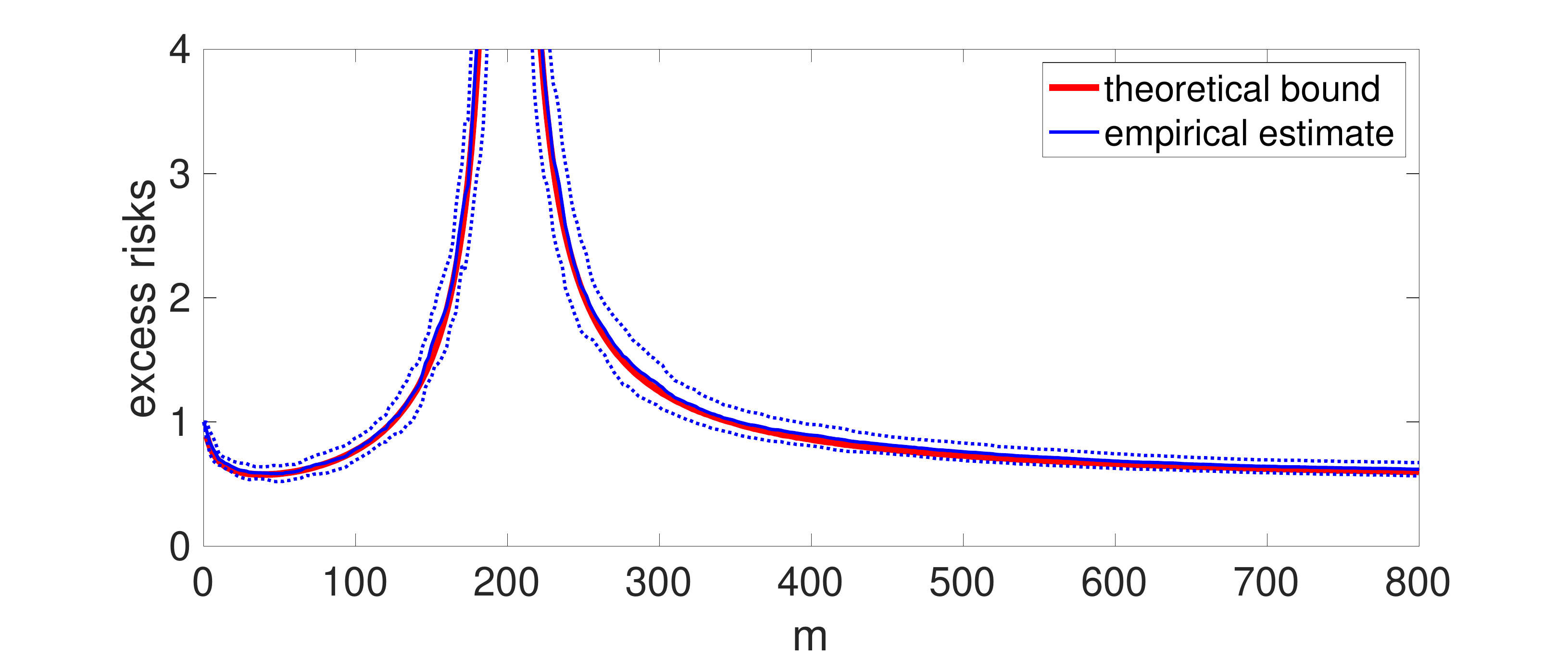}
\end{center}

\vspace*{-.4cm}

\caption{Example of a double descent curve, for linear regression with random projections with $n = 200$ observations, in dimension $d= 400$ and a non-isotropic covariance matrix. The data are normalized so that predicting zero leads to an excess risk of $1$ and the noise so that the optimal expected risk is $1/4$. The empirical estimate is obtained by sampling 20 datasets and 20 different random projections from the same distribution and averaging the corresponding excess risks. We plot the empirical performance together with our asymptotic equivalents from \mysec{rp}.\label{fig:doubledescent}}
\end{figure}

The phenomenon has been theoretically analyzed in several settings, such as random features based on neural networks~\cite{mei2019generalization}, random Fourier features~\cite{liao2020random}, or linear regression~\cite{belkin2020two,hastie2019surprises}.
While the analysis of~\cite{mei2019generalization,liao2020random} for random features corresponds to a single prediction problem with a sequence of increasingly larger prediction models, most of the analysis of \cite{hastie2019surprises} for linear regression does not consider a single problem, but varying problems, which does not actually lead to a double descent curve.
 Random subsampling on a single prediction problem was analyzed with a simpler model with isotropic covariance matrices in~\cite{belkin2020two} and \cite[Section 5.2]{hastie2019surprises}, but without a proper double descent as the model is too simple to account for a U-shaped curve in the under-parameterized regime. In work related to ours, principal component regression was analyzed by~\cite{xu2019number} with a double descent curve but with less general assumptions regarding the spectrum of the covariance matrix and the optimal predictor.

In this paper, we consider linear regression problems and consider \emph{random projections}, whose number increases, where we provably exhibit a double descent curve for a fixed prediction problem. Our analysis follows the high-dimensional analysis of~\cite{hastie2019surprises,dobriban2018high,richards2021asymptotics,lejeune2022asymptotics,bartlett_montanari_rakhlin_2021} based on random matrix theory~\cite{bai2010spectral}, and we give asymptotic expressions for the (squared) bias and the variance terms of the excess risk. These expressions and the trade-offs they lead to will be the same as what can be obtained with ridge regression~\cite{hoerl1970ridge}, where a squared Euclidean penalty is added to the empirical risk.

The paper is organized as follows.
\BIT
\item We first present the asymptotic set-up we will follow in \mysec{highdim}, and review in \mysec{rmt} the results from random matrix theory that we will need for our main result on random projections.  

\item We consider in \mysec{ridge} the ridge regression estimator and re-interpret the results of~\cite{dobriban2018high,richards2021asymptotics,cui2021generalization,wu2020optimal,bartlett_montanari_rakhlin_2021} using classical notions from non-parametric statistics, namely the degrees of freedom, a.k.a.~effective dimensionality~\cite{zhang2005learning,FCM:Caponetto+Vito:2007}. When going from a fixed design analysis (where inputs are assumed deterministic) to a random design analysis (where inputs are random), the prediction performance in terms of bias and variance has the same expression, but with a larger regularization parameter, which corresponds to an additional regularization, which, following~\cite{montanari2022interpolation}, we will refer to as ``self-induced''.

\item With our new interpretation, we consider in \mysec{minnorm} the minimum norm least-squares estimate and analyze its performance (which corresponds to $\lambda=0$ above for ridge regression), thus recovering the results of~\cite{hastie2019surprises,bartlett2020benign}. This corresponds to the end of the double descent curve.

\item In \mysec{rp}, we compute asymptotic equivalents of the generalization performance (in terms of bias and variance) of the minimum norm least-squares fit with random projections, providing simple expressions for the double descent phenomenon.  If $n$ is the number of observations and $m$ is the number of random projections, the variance term goes up and explodes at $m=n$ and then goes down. In contrast, the bias term may exhibit a U-shaped curve on its own in the under-parameterized regime ($m<n$), blows up at $m=n$, and then goes down. Our result relies on using a high-dimensional analysis both on the data and on the random projections. 
\EIT

\section{High-dimensional analysis of linear regression}
\label{sec:highdim}
We consider the traditional random design linear regression model, where $x_1,\dots,x_n \in \rb^d$ are sampled independently and with identical distributions (i.i.d.) with covariance matrix $\S \in \rb^{d \times d}$, and $y_i = x_i^\top \theta_\ast + \varepsilon_i$, with $\varepsilon_i$ and $x_i$ independent, and $\E [\varepsilon_i]=0$ and $\var(\varepsilon_i) = \sigma^2$ for some $\theta_\ast \in \rb^d$.

We denote $y  \in \rb^n$ the response vector, $X \in \rb^{n \times d}$ the design matrix, and $\varepsilon \in \rb^n$ the noise vector. We denote by $\hS = \frac{1}{n} X^\top X \in \rb^{d \times d}$ the non-centered empirical covariance matrix, while $X X^\top \in \rb^{n \times n}$ is the kernel matrix. 

The excess risk for an estimator $\hat{\theta}$ is $\mathcal{R}(\hat \theta) = (\hat{\theta}  - \theta_\ast) \S ( \hat{\theta}  - \theta_\ast) $, and we will always consider expectations with respect to $\varepsilon$, thus conditioned on $X$ and on the potential additional random projections.
 The expectation of the excess risk will be composed of two terms: a (squared) ``bias'' term
$\mathcal{R}^{({\rm bias})}(\hat \theta)$ corresponding to $\sigma=0$ (and thus independent of $\varepsilon$), and a ``variance'' term $\E_\varepsilon \big[ \mathcal{R}^{({\rm var})}(\hat \theta) \big]$ corresponding to $\theta_\ast=0$ (and after taking the expectation with respect to~$\varepsilon$). All of our asymptotic results will then be almost surely in all other random quantities (e.g., $X$ and the random projections $S$ later).

We make similar high-dimensional assumptions as \cite{dobriban2018high,richards2021asymptotics}, that is:
\BIT
\item[\textbf{(A1)}] $X =  Z \S^{1/2}$ with $Z \in \rb^{n\times d}$ with sub-Gaussian i.i.d.~components with mean zero and unit variance.

\item[\textbf{(A2)}] The sample size $n$ and the dimension $d$ go to infinity, with $  \frac{d}{n}$ tending to $\gamma > 0$.

\item[\textbf{(A3)}] The spectral measure 
$  \frac{1}{d} \sum_{i=1}^d   \delta_{\sigma_i}$
of $\S$ converges to a probability distribution $\mu$ on $\rb_+$, where $\sigma_1,\dots,\sigma_d$ are the eigenvalues of $\S$. Moreover, $\mu$ has compact support in $\rb_+^\ast$, and $\S $ is invertible and bounded in operator norm.

\item[\textbf{(A4)}] The measure $   \sum_{i=1}^d ( v_i ^\top \theta_\ast)^2 \delta_{\sigma_i}$ converges to a measure $\nu$ with bounded mass, where $v_i$ is the unit-norm eigenvector of $\S$ associated to $\sigma_i$. The norm of $\theta_\ast$ is bounded.

\EIT

Assumption \textbf{(A1)} does \emph{not} assume Gaussian data but includes $Z$ with standard Gaussian components or Rademacher random variables (uniform in $\{-1,1\}$).

Assumption \textbf{(A2)} states that the ratio of dimensions tends to a constant, but could be relaxed by a uniform boundedness assumption~\cite{rubio2011spectral}. See~\cite{cheng2022dimension} for an analysis that goes beyond this assumption of $n$ and $d$ being of the same order.

Assumption \textbf{(A3)} implies that for any bounded function $r: \rb_+ \to \rb$, $  \frac{1}{d} \tr [ r(\Sigma)] 
\to \int_0^{+\infty}\!\! r(\sigma) d\mu(\sigma)$.
Note that in  \textbf{(A3)}, we assume that the support of the limiting $\mu$ is bounded away from zero (e.g., no vanishing eigenvalues).

Assumption \textbf{(A4)} is equivalent to: for any bounded function $r: \rb_+ \to \rb$, $  \theta_\ast^\top r(\S) \theta_\ast \to \int_0^{+\infty} \!\! r(\sigma) d\nu(\sigma)$.
Moreover, it is often replaced by $\theta_\ast$ being random with mean zero and covariance matrix proportional to identity~\cite{dobriban2018high}, or a spectral variant of $\S$~\cite{richards2021asymptotics}. This corresponds to having $\nu$ having a density with respect to~$\mu$.

\section{Random matrix theory tools}
\label{sec:rmt}

We consider the kernel matrix $XX^\top = Z \S Z^\top \in \rb^{n \times n}$ with all components of $Z \in \rb^{n \times d}$ being i.i.d.~sub-Gaussian with zero mean and unit variance, that is, following Assumption \textbf{(A1)}. We also assume  \textbf{(A2)} and  \textbf{(A3)} throughout this section. We denote by $\hS = \frac{1}{n}X^\top X \in \rb^{d \times d}$ the empirical covariance matrix.

We now present the tools from random matrix theory that we will need. Most of them have already been used in the same context~\cite{dobriban2018high,hastie2019surprises,richards2021asymptotics,lejeune2022asymptotics}, but more refined ones will be needed along the lines of~\cite{dar2021common,lejeune2022asymptotics} (\mysec{spectrace}) and we will give explicit interpretations in terms of degrees of freedom (\mysec{rmt_inter}) and self-induced regularization (\mysec{implicitreg}).

\subsection{Summary and re-interpretation of existing results} 
\label{sec:rmt_inter}

We will need to relate the spectral properties of the empirical covariance matrix $\hS$ to the ones of the population covariance matrix $\S$. This typically includes the distribution of eigenvalues, but in this paper, we will only need spectral functions of the form $\tr [ r(\hS)]$, or more general quantities, such as $\tr [A r(\hS)]$, $\tr [ Ar(\hS) Br(\hS)]$, for matrices $A,B \in \rb^{d \times d}$.

We summarize the relevant results from random matrix theory through the asymptotic equivalence,\footnote{In this paper, we use the asymptotic equivalent notation $u \sim v$, to mean that the ratio $u/v$ tends to one when the dimensions $n,d$ go to infinity. This allows to provide results for diverging quantities which are more easily interpretable, such as degrees of freedom.} for any $\lambda > 0$,
\BEQ
\label{eq:eqdf} \tr \big[ \hS ( \hS+ \lambda \idm)^{-1}\big] \sim \tr \big[ \S \big( \S + \kappa(\lambda) \idm \big)^{-1} \big],
\EEQ
where $\kappa: \rb_+ \to \rb_+$ is an increasing function. Within the analysis of ridge regression, these are often referred to as the ``degrees of freedom''~\cite{FCM:Caponetto+Vito:2007,hastie_GAM}, and denoted\footnote{We use the notation $ {\rm df}_1$ as we will introduce a related notion $ {\rm df}_2$ later.}
$$
\widehat{\rm df}_1(\lambda) = \tr \big[ \hS ( \hS+ \lambda \idm)^{-1}\big]
\ \mbox{ and  }  \  {\rm df}_1(\kappa) = \tr \big[ \S ( \S+ \kappa \idm)^{-1}\big].
$$
In the limit when $d$ tends to infinity,  by definition of $\mu$ in Assumption \textbf{(A3)}, then $   \frac{1}{d} {\rm df}_1(\kappa) \to \int_0^{+\infty} \frac{ \sigma d\mu(\sigma)}{\sigma + \kappa}$, which is strictly decreasing in $\kappa$, with a value of $1$ at $\kappa=0$. Since $\tr \big[ \hS ( \hS+ \lambda \idm)^{-1}\big] \leqslant d$, this asymptotically defines uniquely $\kappa(\lambda)$.

The extra knowledge from random matrix theory will be the self-consistency equation 
$$
\kappa(\lambda) - \lambda  
= \kappa(\lambda) \cdot  \gamma \int_0^{+\infty} \frac{ \sigma d\mu(\sigma)}{\sigma + \kappa},$$
that allows to define $\kappa(\lambda)$,
which we will write equivalently
$$ \kappa(\lambda) - \lambda   \sim
\kappa(\lambda)  \cdot \frac{1}{n}  {\rm df}_1(\kappa(\lambda) ) .
$$
As shown below, for $\lambda$ large, then $\kappa(\lambda) \sim \lambda$. When $\lambda$ tends to zero (which will be the case in classical scenarios where we regularize less as we observe more data),  $\kappa(\lambda)$ will tend to zero only for under-parameterized models ($\gamma<1)$, while for over-parameterized model ($\gamma>1$), it will tend to a constant.

In statistical terms, the degrees of freedom for the empirical covariance matrix correspond to the degrees of freedom of the population covariance matrix with a larger regularization parameter, leading to an additional regularization.

Beyond \eq{eqdf}, we will need asymptotic equivalents for the   quantities
$\tr \big[ A  \hS ( \hS+ \lambda \idm)^{-1}  \big] $
and $
\tr \big[ A  \hS ( \hS+ \lambda \idm)^{-1}  B  \hS ( \hS+ \lambda \idm)^{-1}\big]  $ for matrices $A,B \in \rb^{d \times d}$.  They will be valid when certain quantities for the matrices $A$ and $B$ converge (see Prop.~\ref{prop:spectral} and Prop.~\ref{prop:spectralK} below).

These results recover existing work with $A,B = \idm$ or $\Sigma$~\cite{ledoit2011eigenvectors,dobriban2018high}, and lead to the same formulas as~\cite{dar2021common,lejeune2022asymptotics} obtained with similar assumptions. They are needed for the ridge regression results in \mysec{ridge} and for the random projection results in \mysec{rp}, where they will, for example, be used with $A = \theta_\ast \theta_\ast^\top$.

\subsection{Self-induced regularization}
\label{sec:implicitreg}

We consider the Stieltjes transform of the spectral measure of the kernel matrix $XX^\top \in \rb^{n \times n}$, with $z \in \mathbb{C} \backslash \rb_+$:
$$
\hphi(z) = \frac{1}{n} \tr \Big[ \Big(
\frac{1}{n} XX^\top  - z \idm \Big)^{-1}
\Big]
= \tr \big[ ( XX^\top - nz\idm)^{-1}\big].
$$
This transform is known to fully characterize the spectral distribution of $XX^\top$ (see, e.g., \cite{bai2010spectral} and references therein).
Then for all $z \in \mathbb{C} \backslash \rb_+$, assuming
\textbf{(A1)}, \textbf{(A2)}, and \textbf{(A3)},
$\hphi(z)$ is known to converge almost surely, and its limit $\varphi(z)$ satisfies the following equation (see Appendix~\ref{app:spectral} for a simple argument leading to it)~\cite{bai2010spectral,ledoit2011eigenvectors}:
\BEQ
\label{eq:varphiz}
\frac{1}{\varphi(z)} + z = \gamma \int_0^{+\infty} \frac{ \sigma d\mu(\sigma)}{1 + \sigma \varphi(z)}.
\EEQ
When $\Sigma = \sigma \idm$, this allows to compute $\varphi(z)$ and, by inversion of the Stieltjes transform, to recover the Marchenko-Pastur distribution. In this paper, we will not need to know the limiting density (which is anyway uneasy to describe for general $\Sigma$) and only access it through its Stieltjes transform.

Indeed, for $z = -\lambda$ for $\lambda > 0$, we get $
\hphi(-\lambda) = \tr \big[ (XX ^\top + n \lambda \idm)^{-1} \big] \to \varphi(-\lambda) $ almost surely, with
\BEQ
\label{eq:varphilambda}
\frac{1}{\varphi(-\lambda)} - \lambda = \gamma \int_0^{+\infty} \frac{ \sigma d\mu(\sigma)}{1 + \sigma \varphi(-\lambda)}.
\EEQ

In the ridge regression context, as mentioned above, the quantity ${\rm df}_1(\kappa) = \tr [ \S (\S + \kappa \idm)^{-1} ] \in [0,d]$ is referred to as the ``degrees of freedom''. It is a strictly decreasing function of $\kappa$, with ${\rm df}_{1}(0) = {\rm rank}(\S)$.  It is asymptotically equivalent to $  \sum_{i=1}^d \frac{\sigma_i}{\sigma_i + \kappa}
\sim d \int_0^{+\infty} \frac{\sigma d\mu  (\sigma)}{\sigma + \kappa}$. Thus, we can rewrite \eq{varphilambda} as
$$
\frac{1}{\varphi(-\lambda)} - \lambda \sim \frac{1}{\varphi(-\lambda)}  \cdot \frac{1}{n}  {\rm df}_1\Big(\frac{1}{\varphi(-\lambda)} \Big) .$$
Therefore, we can define our equivalent regularization parameter $\kappa(\lambda)=\frac{1}{\varphi(-\lambda)} \in \rb_+$
 which is the almost sure limit of 
$ 1 / \tr \big[ (XX ^\top + n \lambda \idm)^{-1} \big] $, and  such that
\BEQ
\label{eq:kappa}
\kappa(\lambda) - \lambda \sim \kappa(\lambda)  \cdot \frac{1}{n}  {\rm df}_1(\kappa(\lambda) ) 
\ \ \Leftrightarrow \  \ \lambda \sim \kappa(\lambda) \Big( 1 - \frac{1}{n}  {\rm df}_1(\kappa(\lambda) )  \Big). 
\EEQ

Depending on the relationship between $d$ and $n$ (that is, $d<n$ or $d>n$), we have different behaviors for the function $\kappa$ (see below), but $\kappa(\lambda)$ is always larger than $\lambda$.
This additional regularization has been explored in a number of works~\cite{lejeune2022asymptotics,jacot2020implicit,cheng2022dimension}, and we refer to it as self-induced.

Note that in order to compute $\kappa(\lambda)$, we can either solve \eq{kappa} if we can compute $ {\rm df}_1(\kappa(\lambda))$, or simply use that $\kappa(\lambda)^{-1}$ is the almost sure limit of 
$ \tr \big[ (XX ^\top + n \lambda \idm)^{-1} \big]$, when $n,d$ go to infinity. We now provide properties of the function $\kappa$.
 
\paragraph{Isotropic covariance matrices} We consider the case $\Sigma = \sigma \idm$ to first study the dependence between $\kappa(\lambda)$ and $\lambda$. By the use of Jensen's inequality, this will lead to bounds in the general case. In this isotropic situation, we have $\frac{1}{n} {\rm df}_1(\kappa) = \frac{ \gamma \sigma}{\sigma + \kappa}$, and \eq{kappa} is equivalent to $\lambda = \kappa(\lambda) \big( 1 - \frac{ \gamma \sigma}{\sigma + \kappa}\big) $. We can solve it in closed form as:
\BEQ
\label{eq:kappaiso}
\kappa(\lambda) = \frac{1}{2} \Big( \lambda - \sigma(1-\gamma) + \sqrt{ (\sigma ( 1-\gamma)-\lambda)^2 + 4\lambda \sigma} \Big).
\EEQ
We then have three cases, as illustrated in \myfig{kappa}. The function $\kappa$ is always increasing with the same asymptote $\lambda + \sigma \gamma$ at infinity, but different behaviors at $0$ (see a more thorough discussion in \cite[Section 5.4.1]{lejeune2022asymptotics}):
\BIT
\item $\gamma<1$: $\kappa(0)=0$ with $\kappa'(0) = 1/(1-\gamma)$.
\item $\gamma>1$: $\kappa(0) = (\gamma-1) \sigma >0$.
\item $\gamma=1$: $\kappa(0)=0$ with $\kappa'(0) = +\infty$, and $\kappa \sim \sqrt{\lambda}$ around $0$.
\EIT

\begin{figure}
\begin{center}
\includegraphics[scale=.4]{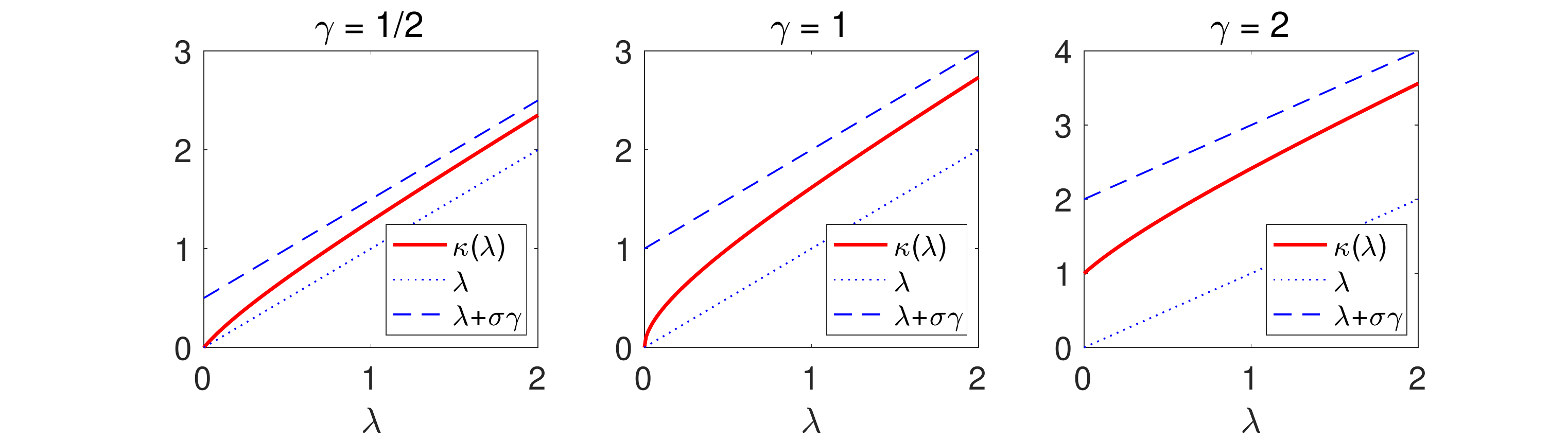}
\end{center}

\vspace*{-.4cm}

\caption{Implicit regularization parameter $\kappa(\lambda)$ in the three regimes for isotropic covariance matrices, with $\sigma = 1$. See text for details. \label{fig:kappa}}
\end{figure}

\paragraph{General case} Beyond isotropic covariance matrices, we have a similar behavior in the general case, in particular, by Jensen's inequality, the expression in \eq{kappaiso} is an upper-bound with $\sigma$ replaced by $\frac{1}{d} \tr(\Sigma)$.
\BIT

\item \emph{Under-parameterized} ($\gamma < 1 \Leftrightarrow d<n$): we then have $ {\rm df}_1(\kappa(\lambda) )\leqslant d < n$, and the function $\lambda \mapsto \kappa(\lambda)$ is strictly increasing with $\kappa(0)=0$ and $\kappa(\lambda) \in [ \lambda, \lambda/(1-d/n)]$, with an equivalent $\kappa(\lambda) \sim \lambda + \frac{1}{n} \tr \Sigma $ when $\lambda$ tends to infinity, and 
the equivalent $\kappa(\lambda) \sim \lambda/(1-d/n)$ when $\lambda$ tends to zero (since we have assumed that ${\rm rank}(\S) = d$).

\item \emph{Over-parameterized} ($\gamma > 1 \Leftrightarrow d>n$): we then have $\kappa(0)>0$, which is defined by 
${\rm df}_1(\kappa(0) )=n$. The function $\lambda \mapsto \kappa(\lambda)$ is still strictly increasing, 
with an equivalent $\kappa(\lambda) \sim \lambda + \frac{1}{n} \tr \Sigma$ when $\lambda$ tends to infinity. By Jensen's inequality, we have 
$  {\rm df}_1(\kappa(\lambda) )  \leqslant \frac{ \tr \S}{\kappa(\lambda) + \tr \S / d} 
\leqslant \frac{ \tr \S}{\kappa(\lambda)  } $. This in turn implies that $\kappa(\lambda) \in [ \lambda, \lambda + \frac{\tr \S}{n}]$, and also a finer bound based on \eq{kappaiso} with $\sigma$ replaced by $\frac{1}{d} \tr (\Sigma)$. Moreover, we have the bound $\kappa(0) \leqslant \frac{\tr \S}{n} ( 1 - n/d) = \frac{ \tr \S}{d} ( \gamma - 1)$.
\EIT
 
 \paragraph{``Classical'' statistical asymptotic behaviors} Within positive-definite kernel methods~\cite{FCM:Caponetto+Vito:2007}, it is common to have infinite-dimensional covariance operators, with a sequence of eigenvalues of the form $\lambda_k = \frac{\tau}{k^\alpha}$ with $\alpha>1$ and $k \geqslant 1$. To make it correspond to the high-dimensional framework with $k\in \{1,\dots,d\}$ with $d$ tending to infinity, we need to rescale the eigenvalues by $d^\alpha$, so that the spectral measure
 is $\widehat{\mu} = \frac{1}{d} \sum_{k=1}^d \delta_{\tau (d/k)^\alpha}$, which converges to the distribution of $\tau / u^\alpha$ for $u$ uniform on $[0,1]$. The support of this distribution is bounded from below, but not from above, and thus does not satisfy our assumptions (but in our simulations, our asymptotic equivalents match the empirical behavior). See~\cite[Section 4.2]{cheng2022dimension} for an analysis that covers explicitly this spectral behavior.
 
 In terms of degrees of freedom, we then have, using the same rescaling by $d^\alpha$, and with the change of variable $v=ud(\kappa/\tau)^{1/\alpha}$:
 $$
 {\rm df}_2(\kappa d^\alpha)
 \! \sim \! d \int_0^1\! \frac{ \tau u^{-\alpha}du}{ \tau u^{-\alpha} + d^\alpha \kappa}
 =  d \int_0^1 \!\!\!\frac{  du}{ 1   + (ud)^\alpha \kappa \tau^{-1}}
\! \sim  \!  (\tau/\kappa)^{1/\alpha} \int_0^{+\infty}  \!\frac{  dv}{ 1 + v^\alpha}.
 $$
We get the usual explosion of degrees of freedom in $\kappa^{-1/\alpha}$~\cite{FCM:Caponetto+Vito:2007}.
 It can then be shown, if our formulas apply, that $\kappa(0) \propto \frac{1}{n^\alpha}$. 
See~\cite{cui2021generalization} for a detailed analysis of the consequences of the ridge regression asymptotic equivalents when such assumptions are made.

\subsection{Asymptotic equivalents for spectral functions} 
\label{sec:spectrace}
Following~\cite{ledoit2011eigenvectors,dobriban2018high}, we 
can provide asymptotic equivalents for quantities depending on the spectrum of $\hS$. We prove in Appendix~\ref{app:spectraltrace} the following result, with two asymptotic equivalents matching the earlier work of~\cite[Lemma 10] {dar2021common} that was obtained for the special case of Gaussian distributions.

\begin{proposition}
\label{prop:spectral}
Assume \textbf{(A1)}, \textbf{(A2)},  \textbf{(A3)}, that $A$ and $B$ are bounded in operator norm, and that
the measures $   \sum_{i=1}^d   v_i ^\top A v_i  \cdot\delta_{\sigma_i}$
and $   \sum_{i=1}^d   v_i ^\top B v_i  \cdot\delta_{\sigma_i}$ converge to measures $\nu_A$ and $\nu_B$ with bounded total variation. Then, for $z \in \mathbb{C} \backslash \rb_+$, with $\varphi(z)$ satisfying \eq{varphiz},
\BEA
\label{eq:trA1}
\tr \big[ A  \hS ( \hS - z \idm)^{-1}  \big] 
& \sim &  \textstyle \tr \big[  A \S \big( \S + \frac{1}{\varphi(z)} \idm \big)^{-1} \big]
\\
\label{eq:trAB1}
\tr \big[ A  \hS ( \hS -z \idm)^{-1}  B  \hS ( \hS -z \idm)^{-1}\big]  
& \sim &  \textstyle \tr \big[  A \S \big( \S +  \frac{1}{\varphi(z)} \idm \big)^{-1} B \S \big( \S +  \frac{1}{\varphi(z)} \idm \big)^{-1} \big]  \\
\notag& & \hspace*{-2cm} \textstyle+ \frac{1}{ \varphi(z)^2}   \tr \big[
    A    \big( \S + \frac{1}{\varphi(z)} \idm \big)^{-2}  \Sigma \big]
     \cdot 
     \tr \big[
    B    \big( \S + \frac{1}{\varphi(z)} \idm \big)^{-2}  \Sigma \big]
    \cdot   \frac{1}{ n -  {\rm df}_2(1/\varphi(z))  }. 
    \EEA
\end{proposition}
\eq{trA1} can formally be seen as the limit
$  
\frac{1}{d} \tr \big[ A  \hS ( \hS - z \idm)^{-1}  \big] \to \int_0^{+\infty} \frac{\sigma d\nu_A(\sigma)}{\sigma + 1/\varphi(z)},
$
and a similar result holds for \eq{trAB1}.
From \eq{trA1} and \eq{trAB1}, as shown in Appendix \ref{app:spectraltrace}, we can also derive results for slightly modified traces, with $\hS  ( \hS - z \idm)^{-1} $ replaced by $ ( \hS - z \idm)^{-1} $, as:
\BEA
\label{eq:trA2}
\tr \big[ A   ( \hS - z \idm)^{-1}  \big] 
& \sim & \textstyle \frac{-1}{z \varphi(z)} \tr \big[ A \big(  \S + \frac{1}{\varphi(z)} \idm \big)^{-1} \big]
\\
\label{eq:trAB2}
\tr \big[ A    ( \hS -z \idm)^{-1}  B    \big( \hS -z \idm \Big)^{-1}\big]  
& \sim &  \textstyle\frac{1}{z^2 \varphi(z)^2}  \tr \big[ A   \big( \S + \frac{1}{\varphi(z)} \idm \big)^{-1}    B   \big( \S + \frac{1}{\varphi(z)} \idm \big)^{-1}    \big] 
 \\
\notag & & \textstyle
     \hspace*{-2cm}+ \frac{1}{z^2 \varphi(z)^2}   \tr \big[
    A    \big( \S + \frac{1}{\varphi(z)} \idm \big)^{-2}  \Sigma \big]
     \cdot 
     \tr \big[
    B    \big( \S + \frac{1}{\varphi(z)} \idm \big)^{-2}  \Sigma \big]
    \cdot   \frac{1}{ n -  {\rm df}_2(1/\varphi(z))  }.   \EEA

\paragraph{Expectation of kernel matrices}
Through the matrix inversion lemma, we have
$\hS  ( \hS - z \idm)^{-1} = X^\top X ( X^\top X - nz \idm)^{-1} = X^\top ( XX^\top - nz \idm) X$, and thus we obtain another set of asymptotic results, where we can replace $\Sigma^{1/2} A \Sigma^{1/2}$ by $A$, matching the earlier results of~\cite[Theorem 4.6]{lejeune2022asymptotics}.

\begin{proposition}
\label{prop:spectralK}
Assume \textbf{(A1)}, \textbf{(A2)},  \textbf{(A3)}, that $A$ and $B$ are bounded in operator norm, and that
the measures $   \sum_{i=1}^d   v_i ^\top A v_i  \cdot\delta_{\sigma_i}$
and $   \sum_{i=1}^d   v_i ^\top B v_i  \cdot\delta_{\sigma_i}$ converge to measures $\nu_A$ and $\nu_B$ with bounded total variation. Then, for $z \in \mathbb{C} \backslash \rb_+$, with $\varphi(z)$ satisfying \eq{varphiz},
\BEA
\label{eq:trA1K}
\tr \big[ A Z^\top   ( Z \Sigma Z^\top \!-\! nz \idm)^{-1}  Z \big] 
& \sim &  \textstyle \tr \big[  A   \big( \S + \frac{1}{\varphi(z)} \idm \big)^{-1} \big]
\\
\notag \tr \big[ A  Z^\top   ( Z \Sigma Z^\top\! -\! nz \idm)^{-1}  Z  B  Z^\top   ( Z \Sigma Z^\top\! - \! nz \idm)^{-1}  Z\big]  
& \sim & \textstyle  \tr \big[  A \big( \S \!+\!  \frac{1}{\varphi(z)} \idm \big)^{-1} B  \big( \S \!+\!  \frac{1}{\varphi(z)} \idm \big)^{-1} \big]  \\
\label{eq:trAB1K}
& & \textstyle \hspace*{-4.5cm}+ \frac{1}{ \varphi(z)^2}   \tr \big[
    A    \big( \S + \frac{1}{\varphi(z)} \idm \big)^{-2}    \big]
     \cdot 
     \tr \big[
    B    \big( \S + \frac{1}{\varphi(z)} \idm \big)^{-2}    \big]
    \cdot   \frac{1}{ n -  {\rm df}_2(1/\varphi(z))  }.
    \EEA
\end{proposition}
  Like in~\cite[Theorem 4.6]{lejeune2022asymptotics}, \eq{trAB1K} can be rewritten more intuitively as
  $$\textstyle 
  Z^\top   ( Z \Sigma Z^\top\! -\! nz \idm)^{-1}  Z  B  Z^\top   ( Z \Sigma Z^\top\! - \! nz \idm)^{-1}  Z \sim
  \big( \S +  \frac{1}{\varphi(z)} \idm \big)^{-1} \big( B + \mu(z)\idm \big)  \big( \S +  \frac{1}{\varphi(z)} \idm \big)^{-1},
  $$
  with $  \mu(z) =  \frac{1}{ \varphi(z)^2}  \frac{\tr  [
    B     ( \S + \frac{1}{\varphi(z)} \idm  )^{-2}     ]}{n -  {\rm df}_2(1/\varphi(z)) }$.
  
 \paragraph{Letting $\lambda \to 0$ for $\gamma > 1$}
 Following arguments from~\cite[Lemma 6.2]{dobriban2018high}, in the high-dimensional situation where $\gamma>1$, we can take the limit $\lambda=0$, with the implicit regularization parameter $\kappa(0)>0$ defined in \mysec{rmt_inter}, which is such that ${\rm df}_1(\kappa(0))=n$.
 This works for the kernel version since we can write
 $$\hS ( \hS - z \idm)^{-1} = X^\top X ( X^\top X - nz \idm)^{-1} = X^\top  ( XX^\top   - nz \idm)^{-1} X
 = \Sigma^{1/2} Z^\top (Z \Sigma Z^\top - nz \idm)^{-1} \Sigma^{1/2},  $$
 which makes sense even with $z=0$, as the kernel matrix $XX^\top$ is then asymptotically almost surely invertible (since $\Sigma$ is invertible, and $ZZ^\top$ almost surely is~\cite{bai2008limit}). This will be used in the over-parameterized regime in \mysec{minnorm} and for random projections in \mysec{rp}.

 \paragraph{Letting $\lambda \to 0$ for $\gamma <1$} In this situation, $\kappa(\lambda)$ tends to zero, and we can use \eq{trA2} and \eq{trAB2} instead, that is, 
  $
\tr \big[ A \big( \hS + \lambda \idm \big)^{-1} \big] 
 \sim   \frac{\kappa(\lambda) }{\lambda  } \tr \big[ A \big(  \S + \kappa(\lambda)\idm \big)^{-1} ],
$ 
with $\frac{\kappa(\lambda) }{\lambda  } \sim \frac{1}{1-\gamma}$ when $\lambda$ goes to zero, and $\kappa(0)=0$, leading to 
 \BEA
\label{eq:traces1cov} \tr \big[ A  \hS  ^{-1} \big] 
& \sim &  \frac{1}{1-d/n  } \tr  [ A    \S  ^{-1} ].
\EEA
Equipped with the proper random matrix theory tools, we can apply them to least-squares regression, starting with ridge regression in \mysec{ridge}, its limit when $\lambda \to 0$ in \mysec{minnorm}, and then with random projections in \mysec{rp}.

\section{Analysis of ridge regression}
\label{sec:ridge}
We consider the ridge regression estimator, obtained as the unique minimizer of $\frac{1}{n} \sum_{i=1}^n ( y_i - x_i^\top \theta)^2 + \lambda \| \theta\|_2^2$, which is equal to:
$$
\hat{\theta} = (X^\top X + n \lambda \idm)^{-1} X^\top y= X^\top (XX^\top  + n \lambda \idm)^{-1}  y.
$$
In the fixed design framework, its analysis is explicit and leads to usual bias/variance trade-offs based on simple quantities.

\subsection{Fixed design analysis of ridge regression}
In the fixed design set-up where inputs $x_1,\dots,x_n$ are assumed deterministic, we obtain an expected excess risk, with $\S$ replaced with~$\hS$, which considerably simplifies the analysis~(see, e.g.,~\cite{hsu2012random}):
$$
\E_\varepsilon \big[ ( \hat{\theta} - \theta_\ast)^\top \hS 
( \hat{\theta} - \theta_\ast) \big]=
\lambda^2 \theta_\ast^\top ( \hS+ \lambda \idm )^{-2}   \hS\theta_\ast + \frac{\sigma^2}{n} \tr \big[   \hS^2 (\hS +\lambda \idm)^{-2} \big].
$$
The (squared) bias term $\lambda^2 \theta_\ast^\top ( \hS+ \lambda \idm )^{-2}   \hS\theta_\ast$ is increasing in $\lambda$, and depends on how the true~$\theta_\ast$ aligns with eigenvectors of $\hS$, and ``source conditions'' are typically  used to characterized this alignment~\cite{FCM:Caponetto+Vito:2007}.

This leads us to introduce the two classical different notions degrees of freedom ${\rm df}_1(\lambda) = \tr\big[ \S ( \S + \lambda\idm)^{-1} \big]$ and ${\rm df}_2(\lambda) =  \tr\big[ \S^2 ( \S + \lambda\idm)^{-2} \big]$ as key quantities~\cite{hsu2012random}. Typically, they behave similarly when $\lambda$ tends to zero (in particular, they are both equal to the rank of $\Sigma$ for $\lambda=0$). We will see in \mysec{minnorm} that when they differ significantly, this has consequences regarding the relevance of the end of the double descent curve.

Our goal is to obtain similar results to those for fixed design, using degrees of freedom and (squared) bias of the form $\lambda^2 \theta_\ast^\top ( \S+ \lambda \idm )^{-2}   \S\theta_\ast $. While bounds can be obtained in expectations~\cite{mourtada2022elementary} or high probability~\cite{FCM:Caponetto+Vito:2007}, we aim here at getting asymptotic equivalents.

\subsection{Random design analysis of ridge regression}
In this section, we recover the results from~\cite{dobriban2018high,mourtada2022elementary,bartlett_montanari_rakhlin_2021} with an explicit interpretation in terms of degrees of freedom.

We have, separating the noise from the part coming from $\theta_\ast$:
 \BEA
 \label{eq:ridgetheta}
 \hat{\theta} & = &  (X^\top X + n \lambda \idm)^{-1} X^\top y 
 = (X^\top X + n \lambda \idm)^{-1} X^\top X \theta_\ast + (X^\top X + n \lambda \idm)^{-1} X^\top \varepsilon.
 \\
\notag & = & 
 (\hS +   \lambda \idm)^{-1} \hS \theta_\ast + (\hS +  \lambda \idm)^{-1} \frac{X^\top \varepsilon}{n}.
 \EEA
 This leads to the following proposition, with the same expressions as~\cite[Theorem 4.13]{bartlett_montanari_rakhlin_2021} (see also~\cite{cheng2022dimension} for the same expressions in a more general context):
\begin{proposition}
Assume \textbf{(A1)}, \textbf{(A2)},  \textbf{(A3)}, and \textbf{(A4)}. For the ridge regression estimator in \eq{ridgetheta}, we have:
 \BEAS
 \E_\varepsilon \big[ \mathcal{R}^{({\rm var})}(\hat{\theta}) \big]
 & \sim & 
  \frac{\sigma^2}{n}  {\rm df}_2(\kappa(\lambda))
 \cdot
  \frac{1}{1 - \frac{1}{n}  {\rm df}_2(\kappa(\lambda))}
\\
 \mathcal{R}^{({\rm bias})}(\hat{\theta}) 
   & \sim & \kappa(\lambda)^2 \theta_\ast ^\top \S ( \S + \kappa(\lambda) \idm)^{-2} \theta_\ast 
 \cdot   \frac{1}{1 - \frac{1}{n}  {\rm df}_2(\kappa(\lambda))},
 \EEAS
with $\kappa(\lambda) $ related to $\lambda$ by $\kappa(\lambda) \big( 1 - \frac{1}{n} {\rm df}_1(\kappa(\lambda)) \big) \sim \lambda$.
\end{proposition} 
 \begin{proof}
The variance term is exactly the same as the one from~\cite{dobriban2018high}, and we simply provide here a reinterpretation with degrees of freedom. We obtain it by taking expectations starting from \eq{ridgetheta} to get
$ \E_\varepsilon \big[ \mathcal{R}^{({\rm var})}(\hat{\theta}) \big]
= \frac{\sigma^2}{n} \tr \big[ \S    (\hS +  \lambda \idm)^{-2} \hS \big]$. We can then use \eq{trA2} and \eq{trAB2} with $A = \idm$, $B = \Sigma$, and $z = - \lambda$, to get, using $\kappa(\lambda) \tr \big[ \S    (\S +  \kappa(\lambda) \idm)^{-2}  \big]
=  {\rm df}_1(\kappa(\lambda) ) -   {\rm df}_2(\kappa(\lambda) )$:
\BEAS
\E_\varepsilon \big[ \mathcal{R}^{({\rm var})}(\hat{\theta}) \big]
& = &\textstyle\frac{\sigma^2}{n} \tr \big[ \S    (\hS +  \lambda \idm)^{-2} \hS \big]
=  \frac{\sigma^2}{n}\tr \big[ \S    (\hS +  \lambda \idm)^{-1}   \big]
- \lambda \frac{\sigma^2}{n} \tr \big[ \S    (\hS +  \lambda \idm)^{-2}   \big] \\
& \sim & \textstyle \frac{\sigma^2}{n}  \frac{ \kappa(\lambda)}{\lambda} \tr \big[\S (\S + \kappa(\lambda)\idm)^{-1} \big]
-    \frac{\sigma^2}{n} \frac{ \kappa(\lambda)^2}{\lambda}   \tr \big[ \S   \big( \S + \kappa(\lambda) \idm \big)^{-2}     \big] 
 \\
  & & \textstyle
   -  \frac{\sigma^2}{n}\frac{ \kappa(\lambda)^2}{\lambda}   \tr \big[
    \S^2    \big( \S +\kappa(\lambda) \idm \big)^{-2}    \big]
     \cdot 
     \tr \big[
   \S    \big( \S + \kappa(\lambda) \idm \big)^{-2}    \big]
    \cdot   \frac{1}{ n -  {\rm df}_2(\kappa(\lambda) )  } \\
    & = & \textstyle\frac{\sigma^2}{n} \frac{ \kappa(\lambda)}{\lambda}  {\rm df}_2(\kappa(\lambda))
       -  \frac{\sigma^2}{n}   \frac{ \kappa(\lambda)^2}{\lambda}       \tr \big[
   \S    \big( \S + \kappa(\lambda) \idm \big)^{-2}    \big]
    \cdot   \frac{ {\rm df}_2(\kappa(\lambda) )}{ n -  {\rm df}_2(\kappa(\lambda) )  } \\
    & = &\textstyle \frac{\sigma^2}{n}\frac{ \kappa(\lambda)}{\lambda} {\rm df}_2(\kappa(\lambda))
       -  \frac{\sigma^2}{n}  \frac{ \kappa(\lambda)}{\lambda}        \big(
         {\rm df}_1(\kappa(\lambda) ) -   {\rm df}_2(\kappa(\lambda) )
       \big)
    \cdot    \frac{ {\rm df}_2(\kappa(\lambda) )}{ n -  {\rm df}_2(\kappa(\lambda) )  } \\
   & = & \textstyle\frac{\sigma^2}{n}\frac{ \kappa(\lambda)}{\lambda}   \frac{ {\rm df}_2(\kappa(\lambda) )
   \big( n - {\rm df}_1(\kappa(\lambda) ) )}{ n -  {\rm df}_2(\kappa(\lambda) )  } =  {\sigma^2}{}   \frac{ {\rm df}_2(\kappa(\lambda) )
   }{ n -  {\rm df}_2(\kappa(\lambda) )  } .
\EEAS
For the bias term, we have:
\BEAS
 \mathcal{R}^{({\rm bias})}(\hat{\theta}) 
 & = & \big\| \S^{1/2} \big( (\hS +   \lambda \idm)^{-1} \hS - \idm
 \big) \theta_\ast \|_2^2
  = \lambda^2 \theta_\ast^\top (\hS +   \lambda \idm)^{-1}\Sigma (\hS +   \lambda \idm)^{-1} \theta_\ast.
\EEAS
We then apply \eq{trAB2} with $A = \Sigma$ and $B = \theta_\ast \theta_\ast^\top$, which applies because of Assumption \textbf{(A4)}, to get:
\BEAS
\mathcal{R}^{({\rm bias})}(\hat{\theta}) 
 & = &\textstyle\kappa(\lambda)^2       \theta_\ast^\top  \big( \S + \kappa(\lambda)\idm \big)^{-2}  \S \theta_\ast   
 \\
 & &  \textstyle\hspace*{1cm} + \kappa(\lambda)^2  \tr \big[
    \S^2    \big( \S +\kappa(\lambda)\idm \big)^{-2}    \big]
     \cdot 
      \theta_\ast^\top    \big( \S + \kappa(\lambda)\idm \big)^{-2}  \Sigma  \theta_\ast
    \cdot   \frac{1}{ n -  {\rm df}_2(\kappa(\lambda))  } \\
    & = &\textstyle\kappa(\lambda)^2       \theta_\ast^\top  \big( \S + \kappa(\lambda)\idm \big)^{-2}  \S \theta_\ast 
   \cdot \big( 1 + \frac{{\rm df}_2(\kappa(\lambda)) }{ n -  {\rm df}_2(\kappa(\lambda))  } \big),\EEAS
   which concludes the proof.
\end{proof}
 
Up to the term $  \frac{1}{1 -    {\rm df}_2(\kappa(\lambda))/n}$, we exactly recover the fixed design analysis for the new larger regularization parameter $\kappa(\lambda)$. Note that in most situations, for the optimal regularization parameter, we usually have $ {\rm df}_1(\kappa(\lambda)) \ll n$ and $ {\rm df}_2(\kappa(\lambda)) \ll n$ so that the exploding term disappears.

We thus see two effects when we go from fixed design to random design: (1) an additional self-induced regularization due to moving from $\lambda$ to $\kappa(\lambda) \geqslant \lambda$, and (2) an explosion of the excess risk if the degrees of freedom get too large.

In the next section, we consider the limit when $\lambda$ tends to zero.

\section{Minimum norm least-square estimation}
\label{sec:minnorm}

The ridge regression estimator converges to the minimum $\ell_2$-norm estimator when $\lambda$ tends to zero. It turns out that this is precisely the estimator found by gradient descent started from zero~\cite{gunasekar2018characterizing}. We consider first the under-parameterized case ($\gamma<1$) and then the over-parameterized one ($\gamma>1$).

\subsection{Under-parameterized regime (ordinary least-squares)}

\label{sec:ols}
When $\gamma < 1$ (that is, $n>d$), then the OLS estimator is 
$\hat{\theta} = (X^\top X)^{-1} X^\top y = (X^\top X)^{-1} X^\top (X \theta_\ast + \varepsilon)
= \theta_\ast + (X^\top X)^{-1} X^\top \varepsilon,$
and thus we have $\mathcal{R}^{({\rm bias})}(\hat{\theta})=0$, and:
$$\textstyle
\E_\varepsilon \big[ \mathcal{R}^{({\rm var})}(\hat{\theta})\big] = \S^2 \tr \big[ X  (X^\top X)^{-1}  \S  (X^\top X)^{-1} X^\top \big] = \frac{\sigma^2}{n} \tr \big[   \S  \hS^{-1} \big].
$$
Using \eq{traces1cov}, we  obtain the classical equivalent $  \sigma^2 \frac{ \gamma}{1-\gamma} \sim \sigma^2 \frac{ d}{n - d}$, as derived, e.g.,  in~\cite{hastie2019surprises}.
Note that for Gaussian data, this is, in fact, almost an equality, that is, 
$\E_\varepsilon \big[ \mathcal{R}^{({\rm var})}(\hat{\theta})\big] = \sigma^2 \frac{ d}{n - d - 1}$ for $n>d+1$.

\subsection{Over-parameterized regime}
\label{sec:ridge_overparam}

We now consider the case $\gamma>1$ (that is, $d>n$).
We can see it as the limit when $\lambda$ tends to zero within ridge regression. This is exactly what was obtained in~\cite{hastie2019surprises} (in a non-asymptotic framework), here with an interpretation in terms of degrees of freedom. We obtain, with $\kappa(0)$ such that $ {\rm df}_1(\kappa(0)) = n$:
\BEAS
 \E_\varepsilon \big[ \mathcal{R}^{({\rm var})}(\hat{\theta}) \big]
   & \sim & 
 \frac{\sigma^2}{n}  {\rm df}_2(\kappa(0))
 \cdot
  \frac{1}{1 - \frac{1}{n}  {\rm df}_2(\kappa(0))}
\\
 \mathcal{R}^{({\rm bias})}(\hat{\theta}) 
    & \sim & \kappa(0)^2 \theta_\ast ^\top \S ( \S + \kappa(0) \idm)^{-2} \theta_\ast 
 \cdot   \frac{1}{1 - \frac{1}{n}  {\rm df}_2(\kappa(0))}.
 \EEAS

Following~\cite{bartlett2020benign,hastie2019surprises}, we can try to understand when the over-parameterized limit with no regularization makes statistical sense, with two questions in mind: (1) does it lead to catastrophic over-fitting? (2) can it lead to a good performance? The answers to these questions will depend on how ${\rm df}_1(\kappa(\lambda)) $ and ${\rm df}_2(\kappa(\lambda)) $ are related. Since ${\rm df}_1(\kappa(\lambda)) =n$, we have ${\rm df}_2(\kappa(\lambda)) \leqslant {\rm df}_1(\kappa(\lambda)) = n$, but how much smaller?

\paragraph{Equivalent degrees of freedom} In many standard situations, the two degrees of freedom are constants away from each other, in particular in the infinite-dimensional cases described at the end of \mysec{implicitreg}. Thus the variance term is proportional to $\sigma^2$, while the bias term is proportional to
$\kappa(\lambda)^2 \theta_\ast ^\top \S ( \S + \kappa(\lambda) \idm)^{-2} \theta_\ast$. There is no catastrophic overfitting, but the variance term cannot go to zero as $n$ tends to infinity, and we cannot expect a good performance when $\sigma$ is far from zero. However, in noiseless problems where $\sigma=0$, the bias term can lead to a better performance than what can be obtained with under-parameterized problems (see also \mysec{rp}).

\paragraph{Unbalanced degrees of freedom} If ${\rm df}_2(\kappa(\lambda)) \ll {\rm df}_1(\kappa(\lambda)) = n$, then the variance term can indeed go to zero when $n$ tends to infinity. This happens only in particular situations thoroughly described by \cite{bartlett2020benign,hastie2019surprises,richards2021asymptotics}.

\section{Random projections}
\label{sec:rp}

We consider a random projection matrix $S \in \rb^{d \times m}$, sampled independently from $X$ with the following assumptions:
\BIT
\item[\textbf{(A5)}] $S \in \rb^{d\times m}$ has sub-Gaussian i.i.d.~components with mean zero and unit variance.
\item[\textbf{(A6)}] The number of projections $m$ tends to infinity with $  \frac{m}{n}$ tending to $\delta > 0$.
\EIT
As for the linear regression assumptions, we do not assume Gaussian random projections, and in all of our experiments, we used Rademacher random variables in $\{-1,1\}$. Given the matrix $S$, we consider projecting each covariate $x \in \rb^d$ to $S^\top x \in \rb^m$. Thus, if $\hat{\eta} \in \rb^m$ is the minimum-norm minimizer of $\| y - X S \eta\|_2^2$, we consider $\hat{\theta} = S \hat{\eta} \in \rb^d$. Note that this is different from applying the random projection on the left of $y$ and $X$, which is often referred to as ``sketching''~\cite{dobriban2019asymptotics,raskutti2016statistical}.

The asymptotic performance can be characterized as follows (again, apart from the expectation with respect to the noise variable $\varepsilon$, all results are meant almost surely).

\begin{proposition}
Assume \textbf{(A1)}, \textbf{(A2)},  \textbf{(A3)}, \textbf{(A4)}, \textbf{(A5)},  \textbf{(A6)}. For the minimum norm least-squares estimator $\hat{\theta}$ based on random projections, we have for the under-parameterized regime ($m<n$):
\BEAS
\E_\varepsilon \big[ \mathcal{R}^{^{\rm var}} (\hat{\theta})\big]
& \sim &  \frac{\sigma^2 m}{n-m} = \frac{1}{1 - \frac{m}{n}}  \cdot \frac{\sigma^2 m}{n}  
\\ 
\mathcal{R}^{^{\rm bias}} (\hat{\theta})  
& \sim & \frac{1}{1 - \frac{m}{n}} \cdot \kappa_m \theta_\ast^\top \S ( \S + \kappa_m \idm)^{-1} \theta_\ast,
\EEAS
 with $\kappa_m$ defined by   ${\rm df_1}(\kappa_m) \sim m$. In the over-parameterized regime, 
  we get, for  $\kappa_n$ such that ${\rm df}_1(\kappa_n) \sim n$:
  \BEAS
\E_\varepsilon \big[ \mathcal{R}^{^{\rm var}} (\hat{\theta}) \big]
& \sim & \frac{\sigma^2}{n} \cdot \frac{ {\rm df}_2(\kappa_n)   }{1  -  \frac{1}{n}{\rm df}_2(\kappa_n) }
+  \sigma^2 \frac{n}{m-n}   
\\
 \mathcal{R}^{^{\rm bias}} (\hat{\theta}) 
& \sim & \kappa_n^2 \theta_\ast^\top \S (  \S + \kappa_n \idm)^{-2} \theta_\ast \cdot 
 \frac{ 1 }{ 1 - \frac{1}{n} {\rm df}_2(\kappa_n)}
 + \kappa_n \theta_\ast^\top \S ( \S+ \kappa_n\idm)^{-1} \theta_\ast \cdot \frac{n}{m-n} .
\EEAS
\end{proposition}
\begin{proof}
 We will consider the $\ell_2$-regularized estimator, with a regularization parameter $\lambda$ that we will let go to zero. The validity of such limits follows from the same arguments as~\cite[Lemma 6.2]{dobriban2018high}.
 We thus consider:
\BEAS
\hat{\theta} & = &  S ( S^\top X^\top X S + n \lambda \idm)^{-1} S^\top X^\top y\\
& = &  S ( S^\top X^\top X S + n \lambda \idm)^{-1} S^\top X^\top X\theta_\ast +
S ( S^\top X^\top X S + n \lambda \idm )^{-1} S^\top X^\top \varepsilon \\
& = &  M \theta_\ast +
S ( S^\top X^\top X S + n \lambda \idm )^{-1} S^\top X^\top \varepsilon,
\EEAS
with $M =  S ( S^\top X^\top X S + n\lambda \idm)^{-1} S^\top X^\top X$.

Conditioned on $S$ and $X$, the expected risk is equal to, for the variance part:
\BEA
\notag \E_{ \varepsilon} \big[ \mathcal{R}^{({\rm var})} (\hat{\theta})\big]
& \!\!= \!\!& 
 \sigma^2 \tr \big[ XS ( S^\top X^\top X S + n\lambda \idm )^{-1}  S^\top \S  S ( S^\top X^\top X S  + n\lambda \idm)^{-1} S^\top X^\top \big]
 \\
 \label{eq:varproof}  &= &   \sigma^2 \tr \big[   S^\top \S  S ( S^\top X^\top X S  + n\lambda \idm)^{-2}   S^\top X^\top X S \big]  ,
 \EEA
while, for the bias, we have:
\BEA
  \mathcal{R}^{({\rm bias})} (\hat{\theta})
 \notag & = &  \big(    M \theta_\ast - \theta_\ast \big)^\top \S  \big(    M \theta_\ast - \theta_\ast \big) 
  =  \theta_\ast ^\top \S \theta_\ast
+  
    \theta_\ast^\top M^\top \S  M \theta_\ast  
- 2 
    \theta_\ast ^\top M^\top \S  \theta_\ast  \\
  \label{eq:biasproof}  & = & 
   \theta_\ast ^\top \S \theta_\ast
   - 2 
    \theta_\ast ^\top     X^\top X S ( S^\top X^\top X S  + n\lambda \idm )^{-1} S^\top 
 \S  \theta_\ast 
 \\
\notag  & & \hspace*{1.5cm} 
 +  
    \theta_\ast^\top 
    X^\top X S ( S^\top X^\top X S + n\lambda \idm )^{-1} S^\top    
     \S  S ( S^\top X^\top X S  + n\lambda \idm )^{-1} S^\top X^\top X \theta_\ast   
  .
  \EEA

For the proof, we separate the two regimes $m<n$ and $m<n$. For both of them, we provide asymptotic expansions in two steps, first with respect to $X$ and then $S$ in the under-parameterized regime and vice-versa for the over-parameterized regime.

\paragraph{Under-parameterized regime: expansion with respect to $X$}
We consider $S$ fixed and use the random matrix theory arguments from \mysec{rmt} for~$X$.
We have  a covariance matrix $S^\top \Sigma S \in \rb^{m \times m}$ of rank $m$, so under-parameterized results apply, and we get
for the variance term (first term above), for $S$ fixed, where we can directly consider $\lambda=0$ (because of cancellations):
 $$\textstyle
\E_\varepsilon \big[\mathcal{R}^{(\rm var)}(\hat{\theta}) \big] =  \sigma^2\tr \big(   S^\top \S  S ( S^\top X^\top X S )^{-1}  \big)
\sim  \frac{\sigma^2}{n-m} \tr \big(   S^\top \S  S ( S^\top \Sigma S)^{-1}  \big) = \frac{\sigma^2m}{m-n},
 $$
 independently of the sketching matrix $S$. Note here that $S^\top \Sigma S$ is a random kernel matrix satisfying assumptions of \mysec{rmt}; thus, its spectral measure has a limit.
 
For the bias term, the computation is more involved. 
With $T = \S^{1/2}S$, and $X = \S^{1/2} Z$, it is equal to:
\BEAS
\mathcal{R}^{(\rm bias)}(\hat{\theta}) &= &\theta_\ast ^\top \S \theta_\ast
   - 2 
    \theta_\ast ^\top     \S^{1/2} Z^\top Z T ( T^\top  Z^\top Z T + n\lambda \idm )^{-1} T^\top 
 \S^{1/2}  \theta_\ast
 \\
 & & \hspace*{.5cm} 
 +  
    \theta_\ast^\top 
    \S^{1/2} Z^\top Z T ( T^\top   Z^\top Z T + n\lambda \idm )^{-1} T^\top    
    T ( T^\top   Z^\top Z T + n\lambda \idm )^{-1} T^\top Z^\top Z \S^{1/2} \theta_\ast   .
\EEAS
Using the matrix inversion lemma, we get:
 \BEAS
\mathcal{R}^{(\rm bias)}(\hat{\theta}) & = &\theta_\ast ^\top \S \theta_\ast
   - 2 
    \theta_\ast ^\top     \S^{1/2} {Z^\top  ( Z T T^\top  Z^\top  + n \lambda \idm )^{-1} Z} TT^\top 
 \S^{1/2}  \theta_\ast
 \\
 & & \hspace*{.5cm} 
 +  
    \theta_\ast^\top 
    \S^{1/2} {Z^\top  (  Z T T^\top   Z^\top + n \lambda \idm )^{-1} Z T T^\top    
    T  T^\top Z^\top ( Z T  T^\top   Z^\top + n \lambda \idm)^{-1} Z} \S^{1/2} \theta_\ast  .
      \EEAS
Denoting  $C = TT^\top$, we then have
 \BEAS
\mathcal{R}^{(\rm bias)}(\hat{\theta}) & = &\theta_\ast ^\top \S \theta_\ast
   - 2 
    \theta_\ast ^\top     \S^{1/2}  \textcolor{red}{  Z^\top  ( Z C  Z^\top  + n \lambda \idm )^{-1} Z} C  
 \S^{1/2}  \theta_\ast
 \\
 & & \hspace*{1.5cm} 
 +  
    \theta_\ast^\top 
    \S^{1/2}  \textcolor{red}{ Z^\top  (  Z  C   Z^\top + n \lambda \idm )^{-1} Z C^2 Z^\top ( Z C  Z^\top + n \lambda \idm)^{-1} Z  }  \S^{1/2} \theta_\ast  .
      \EEAS
To find expansions of the red terms above, we can directly use the results from \mysec{spectrace}, using \eq{trA1K} with $A = C\Sigma^{1/2} \theta_\ast \theta_\ast^\top \Sigma^{1/2}$, and \eq{trAB1K}  with $A=\Sigma^{1/2} \theta_\ast \theta_\ast^\top \Sigma^{1/2}$ and $B = C^2$, with the covariance matrix $C$, and thus with degrees of freedom and the implicit regularization parameter $\tilde{\kappa}(\lambda)$ associated to $C$.\footnote{We use a different notation with $\tilde{}$ , to avoid confusion with the same quantities with $\Sigma$.} We can apply Prop.~\ref{prop:spectralK} since $C = TT^\top = \Sigma^{1/2} SS^\top \Sigma^{1/2}$ has almost surely a limiting spectral measure and the resulting needed traces involving the matrices $A$ and $B$ have well-defined limits. We get:
\BEAS
\mathcal{R}^{(\rm bias)}(\hat{\theta})  & \sim &\theta_\ast ^\top \S \theta_\ast
   - 2 
    \theta_\ast ^\top     \S^{1/2}     (  C    + \tilde{\kappa}(\lambda) \idm )^{-1} C
 \S^{1/2}  \theta_\ast
 \\
 & & \hspace*{2cm} 
 +  
    \theta_\ast^\top 
    \S^{1/2}  ( C     + \tilde{\kappa}(\lambda) \idm )^{-1}   C 
    C (  C     + \tilde{\kappa}(\lambda) \idm )^{-1}   \S^{1/2} \theta_\ast    \\
 & & \hspace*{2cm} 
 +  
\tilde{\kappa}(\lambda)^2    
   \frac{ \theta_\ast^\top  \S^{1/2} 
 (   C  + \tilde{\kappa}(\lambda) \idm )^{-2}  \S^{1/2}\theta_\ast
 \cdot \tr \big[  C^2    ( C  + \tilde{\kappa}(\lambda) \idm)^{-2} \big]}{n - \widetilde{{\rm df}}_2(\tilde{\kappa}(\lambda))}
\\
& \sim &\theta_\ast ^\top \S^{1/2}
\big( \idm -C (  C     + \tilde{\kappa}(\lambda) \idm )^{-1}  \big)^2
\S^{1/2} \theta_\ast \\
  & & \hspace*{2cm} 
 +  
\tilde{\kappa}(\lambda)^2    
   \frac{ \theta_\ast^\top  \S^{1/2}
 (  C  + \tilde{\kappa}(\lambda) \idm )^{-2} \S^{1/2}\theta_\ast
 \cdot \tr \big[ C^2   ( C  + \tilde{\kappa}(\lambda) \idm)^{-2} \big]}{n - \widetilde{{\rm df}}_2(\tilde{\kappa}(\lambda))}.
 \EEAS
 When $\lambda$ goes to zero, we have $\tilde{\kappa}(\lambda) \to 0$, $\widetilde{{\rm df}}_2(\tilde{\kappa}(\lambda)) \to m$, as well as
 $C (  C     + \tilde{\kappa}(\lambda)\idm )^{-1} = TT^\top ( TT^\top +  \tilde{\kappa}(\lambda)  \idm)^{-1}
 = T (  T^\top T +  \tilde{\kappa}(\lambda)  \idm)^{-1} T^\top  \to \S^{1/2} S (S^\top \S S)^{-1} S^\top \S^{1/2}$,
 and $\tilde{\kappa}(\lambda) 
 (  C     + \tilde{\kappa}(\lambda)\idm )^{-1} = \idm - C (  C     + \tilde{\kappa}(\lambda)\idm )^{-1} \to \idm 
 - \S^{1/2} S (S^\top \S S)^{-1} S^\top \S^{1/2}$. This leads to:
\BEA
\notag  \mathcal{R}^{(\rm bias)}(\hat{\theta}) & \sim   & 
  \theta_\ast ^\top \S^{1/2}
\big( \idm -  \S^{1/2} S (S^\top \S S)^{-1} S^\top \S^{1/2}    \big)^2
\S^{1/2} \theta_\ast \\
\notag & & \hspace*{4cm} 
 +  
   \frac{ \theta_\ast^\top  \S^{1/2}
 (  \idm - T  (T^\top T  )^{-1} T^\top)^{2} \S^{1/2}\theta_\ast
 \cdot m}{n-m}
\\
\label{eq:biasA} &  =&    \theta_\ast^\top
\big(
\S\! -\! \S S (S^\top \S S)^{-1} S^\top \S
\big)
\theta_\ast \cdot \big( 1 \!+\! \frac{m}{n\!-\!m} \big)
  .
\EEA

\paragraph{Under-parameterized regime: full expansion}

Using results from \mysec{spectrace}, this time with $Z = S^\top$ and the covariance matrix $\Sigma$, with $\kappa_m $ defined by  ${\rm df_1}(\kappa_m) = m$,  we get from Prop.~\ref{prop:spectralK} the equivalent
$\S - \S S (S^\top \S S)^{-1} S^\top \S \sim \S - \S^{1/2} ( \S + \kappa_m \idm)^{-1} \S^{1/2}$, and thus, from \eq{biasA}, we get the desired result:
$
 \mathcal{R}^{^{\rm bias}} (\hat{\theta}) 
  \sim   \frac{1}{1 - m/n}  \kappa_m \theta_\ast^\top \S ( \S + \kappa_m \idm)^{-1} \theta_\ast.
  $
  
\paragraph{Over-parameterized regime: expansion with respect to $S$}
We have, from \eq{varproof}:
\BEAS
\E_{ \varepsilon} \big[ \mathcal{R}^{({\rm var})} (\hat{\theta})\big]
& = & 
   \frac{\sigma^2}{n}   \tr\big(  \S  \textcolor{red}{S (S^\top \hS S + \lambda \idm)^{-1}S^\top\hS S 
( S^\top \hS X S +  \lambda \idm )^{-1} S^\top} \big) .
\EEAS
To obtain an expansion of the red term, we can use Prop.~\ref{prop:spectralK} with covariance matrix
$\hS$ and thus degrees of freedom and $\tilde{\kappa}$ associated to $\hS$:
\BEAS
\E_{ \varepsilon} \big[ \mathcal{R}^{({\rm var})} (\hat{\theta})\big]
& \sim& \frac{\sigma^2}{n} \tr \big[ \S \hS ( \hS    + \tilde{\kappa}(\lambda) \idm)^{-2} \big]
+\frac{\sigma^2}{n}   \tilde{\kappa}(\lambda   ) ^2 
\frac{ \tr \big[ \S  ( \hS + \tilde{\kappa}(\lambda  )  \idm)^{-2}   \big] \cdot  \tr \big[ \hS  (\hS + \tilde{\kappa}(\lambda )  \idm)^{-2}    \big] }{m - \widetilde{{\rm df}}_2(\tilde{\kappa}(\lambda))}.
\EEAS
Using that $\tilde{\kappa}(\lambda) \to 0$ when $\lambda\to 0$,
$\hS ( \hS    + \tilde{\kappa}(\lambda) \idm)^{-2}  = n X^\top X ( X^\top X + n \tilde{\kappa}(\lambda) \idm)^{-2}
$ can be rewritten as $ n X^\top (XX^\top + n \tilde{\kappa}(\lambda) \idm)^{-2} X \to n X^\top ( XX^\top)^{-2} X$,
and $\tilde{\kappa}(\lambda)^2 
 ( \hS    + \tilde{\kappa}(\lambda) \idm)^{-2} = n^2 \tilde{\kappa}(\lambda)^2 (X^\top X + n \tilde{\kappa}(\lambda) \idm)^{-2}
 = ( \idm - X^\top (XX^\top + n \tilde{\kappa}(\lambda) \idm)^{-1} X)^2 \to ( \idm - X^\top (XX^\top)^{-1} X)^2
 =  ( \idm - X^\top (XX^\top)^{-1} X) $, we thus get:
\BEQ
\label{eq:varB} \E_{ \varepsilon} \big[ \mathcal{R}^{({\rm var})} (\hat{\theta})\big]
  \sim  \sigma^2 \tr \big[ \S X^\top \! ( XX^\top  )^{-2}  X \big]
+ \frac{ \tr \big[ \S  ( \idm\! -\! X^\top (XX^\top)^{-1} X) \big] \cdot  \tr \big[   ( XX^\top  )^{-1}    \big] }{m - n}.
\!\!
 \EEQ
We can now take care of the (squared) bias term with the same technique, with $\lambda \to 0$, starting from \eq{biasproof}:
\BEA
\notag
\mathcal{R}^{({\rm bias})}(\hat \theta) & = & \theta_\ast ^\top \S \theta_\ast
- 2 \theta_\ast^\top \S^{1/2} \textcolor{red}{S
( S^\top \hS S+  \lambda \idm )^{-1}  S^\top} \hS   \theta_\ast \\
\notag & & \hspace*{2cm} + \theta_\ast^\top
  \hS  \textcolor{red}{S  (S^\top \hS S + \lambda \idm )^{-1}   S^\top 
\S
S  ( S^\top X^\top X S+ n\lambda \idm)^{-1} S^\top} \hS
\theta_\ast 
\\
\notag & \sim & \theta_\ast ^\top \S \theta_\ast
- 2 \theta_\ast^\top \S^{1/2} 
(  \hS + \tilde{\kappa}(\lambda ) \idm )^{-1}   \hS   \theta_\ast  + \theta_\ast^\top
  \hS (  \hS + \tilde{\kappa}(\lambda  ) \idm )^{-1} 
\S
(  \hS + \tilde{\kappa}(\lambda  ) \idm )^{-1}  \hS
\theta_\ast \\
\notag&&  + \tilde{\kappa}(\lambda  ) ^2 
\frac{ \tr \big[ \S  ( \hS + \tilde{\kappa}(\lambda )  \idm)^{-2}   \big] \cdot \theta_\ast^\top \hS  ( \hS + \tilde{\kappa}(\lambda )  \idm)^{-2}    \hS \theta_\ast }{m -  \widetilde{{\rm df}}_2(\tilde{\kappa}(\lambda))}
\\ 
\notag & \sim & \big\| \S^{1/2} \big(
\idm - X^\top (  XX^\top   + \tilde{\kappa}(\lambda  ) \idm )^{-1} X  \big) \theta_\ast\big\|_2^2
\\
\notag & &\hspace*{2cm}  + \tilde{\kappa}(\lambda   ) ^2 
\frac{ \tr \big[ \S  ( X^\top X + \tilde{\kappa}(\lambda )  \idm)^{-2}   \big] \cdot \theta_\ast^\top X^\top X  ( X^\top X + \tilde{\kappa}(\lambda  )  \idm)^{-2}    X^\top X \theta_\ast }{ m - \widetilde{{\rm df}}_2(\tilde{\kappa}(\lambda))}
\\ 
\notag & \sim  & 
\theta_\ast^\top ( \idm - X^\top (XX^\top)^{-1}  X) \S  ( \idm - X^\top (XX^\top)^{-1}  X) \theta_\ast
  \\
 \label{eq:biasC} & & \hspace*{3cm} + \frac{1}{m-n} \theta_\ast^\top 
X^\top ( X  X^\top)^{-1}  X  \theta_\ast     \cdot \tr\big[ \S  ( \idm - X^\top (XX^\top)^{-1} X)\big] .
\EEA
 
\paragraph{Over-parameterized regime: full expansion}
For $ \kappa_n $ defined as ${\rm df}_1(\kappa_n) = n$ for the full covariance matrix $\Sigma$ (which is exactly the value of $\kappa_m$ above for $m=n$), we get, using Prop.~\ref{prop:spectralK}, with \eq{varB} and \eq{biasC}:
\BEAS
\E \big[ \mathcal{R}^{^{\rm var}} (\hat{\theta}) \big]
& \sim & \sigma^2 \frac{ {\rm df}_2(\kappa_n) }{ {\rm df}_1(\kappa_n) -  {\rm df}_2(\kappa_n)}
+  \sigma^2 \frac{n}{m-n}  
\\
 \mathcal{R}^{^{\rm bias}} (\hat{\theta})  
& \sim & \kappa_n^2 \theta_\ast^\top \S (  \S + \kappa_n\idm)^{-2} \theta_\ast \cdot 
 \frac{ {\rm df}_1(\kappa_n) }{ {\rm df}_1(\kappa_n) -  {\rm df}_2(\kappa_n)}
 +  \kappa_n \theta_\ast^\top \S ( \S+\kappa_n \idm)^{-1} \theta_\ast \cdot \frac{n}{m-n},
\EEAS
 which is the desired result.
\end{proof}
We can make the following observations:
\BIT
\item In the under-parameterized regime, we recover the traditional bias and variance terms divided by $ 1 - \frac{m}{n} $, which leads to the expected catastrophic over-fitting when $m$ is close to $n$. Moreover, while the variance term goes up from $m=0$ to $m=n$, the bias term has one decreasing term
$\kappa_m \theta_\ast^\top \S ( \S + \kappa_m \idm)^{-1} \theta_\ast$ and one increasing term $\big(1 - \frac{m}{n}\big)^{-1}$. In some cases (e.g., for $\theta_\ast$ and $\Sigma$ isotropic), the overall performance always goes up, but in many situations, we obtain the traditional U-shaped curve in the under-parameterized regime.

\item In the over-parameterized regime, the limit when $m$ tends to infinity is exactly the same as the limit $\lambda$ tending to zero for ridge regression in \mysec{ridge_overparam}, since $\kappa_n$ is exactly what was referred to as $\kappa(0)$. Moreover, we have, for both variance and bias, a decreasing function of $m$. Thus, once in this regime, it is always best to take $m$ as large as possible. Note that to achieve the performance for $m=\infty$, we can simply take $\hat{\theta} = X^\top ( XX^\top )^{-1} y$, and there is no need to solve a problem in dimension $m$ with~$m$ large.

\item Combining the two regimes, we indeed see an actual double descent in many scenarios. See illustrative experiments in \mysec{experiments}.
 \EIT

\section{Experiments}
\label{sec:experiments}

In this section, we present illustrative experiments to showcase our asymptotic equivalents from \mysec{rp}.\footnote{Matlab code to reproduce figures can be downloaded from {\small \url{https://www.di.ens.fr/~fbach/dd_rp.zip}}.}
 
\paragraph{Testing the asymptotic limit}
We consider a fixed spectral measure $\mu = \pi_1 \delta_{\sigma_1} + \pi_2 \delta_{\sigma_2}$ already considered by \cite{hastie2019surprises,richards2021asymptotics} and the fixed measure $\nu = \mu$ for the optimal predictor, for which we can compute all of the asymptotic equivalents in \mysec{rp} in closed form. We take $\gamma = d/n = 2$ and plot bias and variance as functions of $\delta = m/n$. We then compare them to experiments with finite $n$ (and the corresponding $d= \gamma n$ and $m = \delta n$), where we sample $\theta_\ast$ and $\Sigma$ from their distributions (with a matrix of eigenvectors uniformly at random in the set of orthogonal matrices). We have here, for $\delta \in [0,1]$,
$$
\kappa(\delta) = \frac{1}{2} \Big( \frac{\gamma}{\delta} ( \pi_1 \sigma_1 + \pi_2 \sigma_2) - \sigma_1 - \sigma_2 
+ \Big[(\frac{\gamma}{\delta} ( \pi_1 \sigma_1 + \pi_2 \sigma_2) - \sigma_1 - \sigma_2 )^2 + 4\sigma_1\sigma_2 ( \frac{\gamma}{\delta}-1)\Big]^{1/2} \Big).
$$
In \myfig{convergence_3examples}, we  can see that as $n$ gets larger, each realization of the experiment tends to the asymptotic limit, illustrating almost sure convergence (which we conjecture to be of order $O(1/\sqrt{n})$), while, when we consider expectations with respect to several realizations, we get a faster convergence (which we conjecture to be of order $O(1/ {n})$).

 \begin{figure}
\begin{center}
\includegraphics[scale=.4]{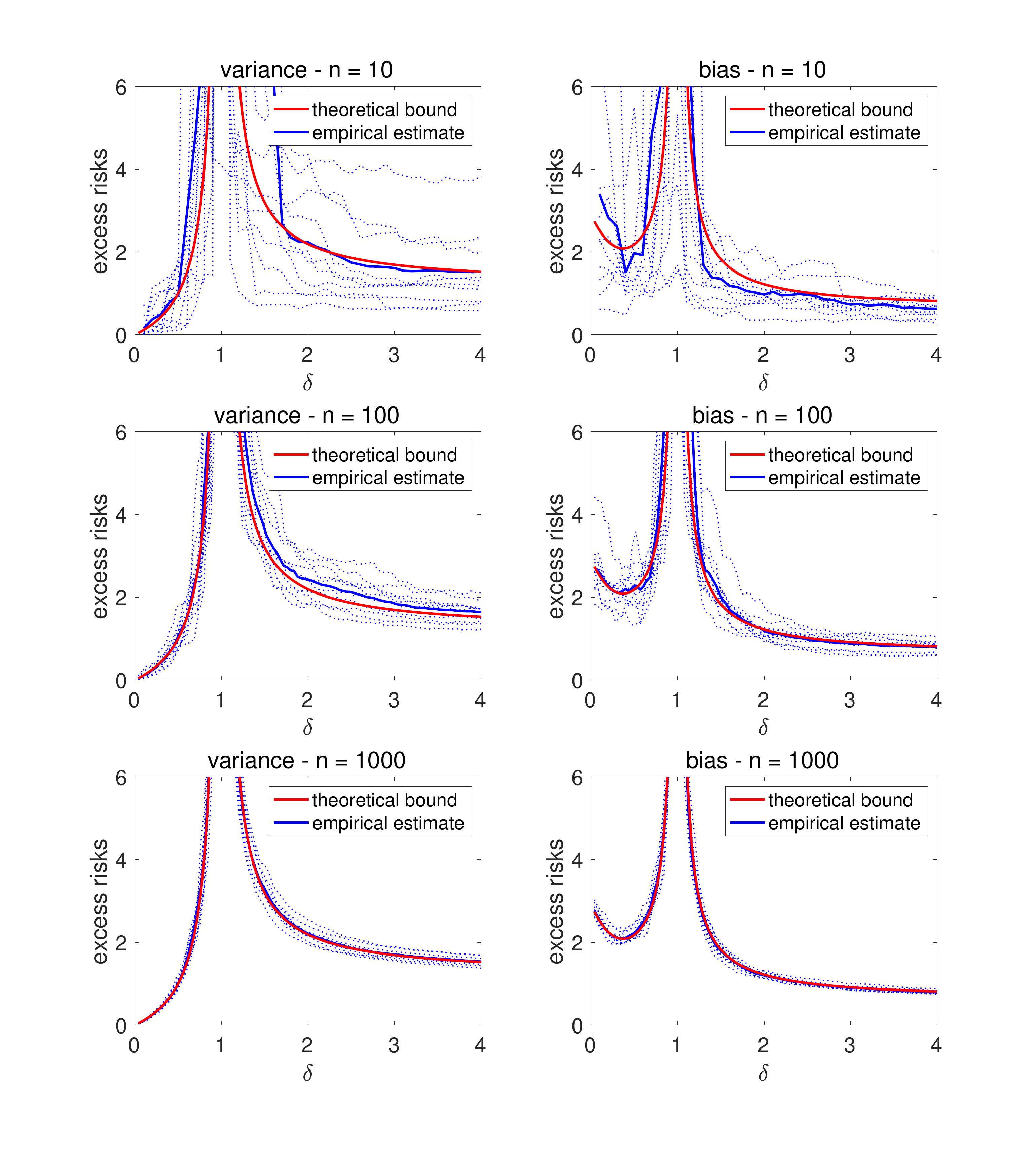}
\end{center}

\vspace*{-1.3cm}

\caption{Comparison of theoretical bounds and empirical estimates for a spectral measure with two Diracs (see text for details): (left) variance, (right) bias, with three different numbers of observations, with $n=10$ (top), $n=100$ (middle), and $n=1000$ (bottom). We plot ten realizations with the same spectral properties, as well as the average excess risk.
\label{fig:convergence_3examples}}
\end{figure}

\paragraph{Illustration of the double descent phenomenon}
We consider a fixed covariance matrix $\Sigma$ of size~$d$, with uniformly random eigenvectors and eigenvalues proportional to $1/k$, for $k \in \{1,\dots,d\}$ (non-isotropic), or constant (isotropic).  We normalize the matrix so that $\tr(\Sigma)=1$. We generate a vector $\theta_\ast \in \rb^d$ from a standard Gaussian distribution and then normalize it so that $\theta_\ast^\top \Sigma \theta_\ast = 1$. Given this unique prediction problem, we generate 40 replications of $Z$ and $S$ from Rademacher random variables and plot the empirical performance for the bias and the variance. For the bounds, we compute $\kappa_m$ from $\kappa_m^{-1} = \E \big[ \tr[ (S^\top \Sigma S)^{-1} ] \big]$, using an average over 40 replications.

In \myfig{bias_variance}, we show the results for the non-isotropic covariance matrix, where we see a U-shaped curve for the bias term. In contrast, in \myfig{bias_variance_no_ushaped}, we show the results for the isotropic covariance matrix, where we do not see a U-shaped curve for the bias term (and thus, there cannot be a U-shaped curve when summing  bias and variance). The asymptotic limits from \mysec{rp} closely match the empirical behavior in both cases.

\begin{figure}
\begin{center}
\includegraphics[scale=.4]{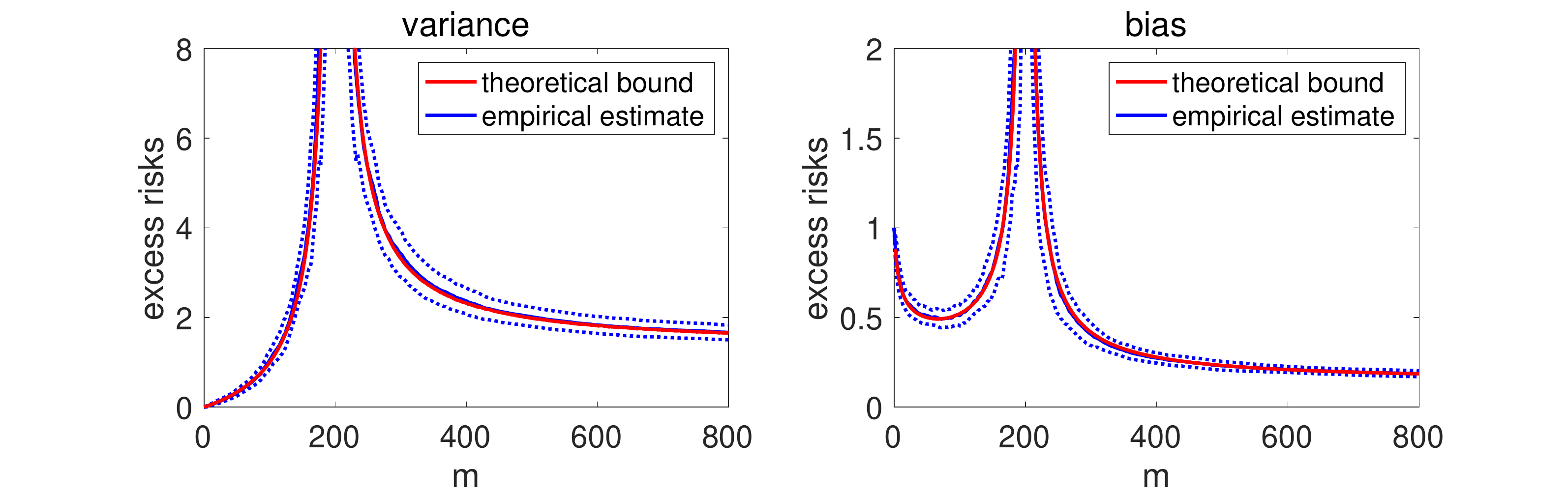}
\end{center}

\vspace*{-.4cm}

\caption{(Left) Variance with $\sigma=1$ and $\tr(\S)=1$. 
 (Right) Bias with $\theta_\ast^\top \S \theta_\ast= 1$. We consider $n=200$, $d=400$, with $Z$ and $S$ sampled from Rademacher random variables, and eigenvalues of $\S$ proportional to $1/k$.
 For the empirical curve, we plot the average performance over 40 replications as well as the standard deviation in dotted. \label{fig:bias_variance}}
\end{figure}

\begin{figure}
\begin{center}
\includegraphics[scale=.4]{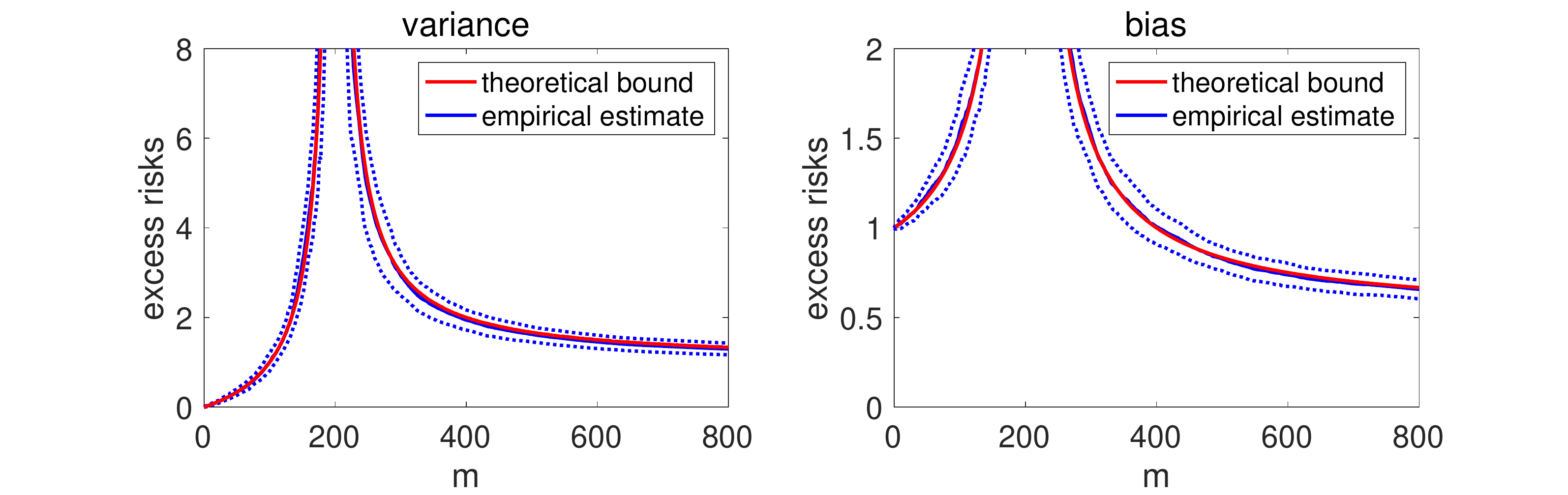}
\end{center}

\vspace*{-.4cm}

\caption{(Left) Variance with $\sigma=1$ and $\tr(\S)=1$. 
 (Right) Bias with $\theta_\ast^\top \S \theta_\ast= 1$. We consider $n=200$, $d=400$, with $Z$ and $S$ sampled from Rademacher random variables and uniform eigenvalues for $\S$.  For the empirical curve, we plot the average performance over 40 replications as well as the standard deviation in dotted. \label{fig:bias_variance_no_ushaped}}
\end{figure}

\section{Conclusion}
In this paper, we have provided a high-dimensional asymptotic analysis of the double descent phenomenon for random projections. This was done using an interpretation of random matrix theory results for empirical covariance matrices based on degrees of freedom. Several avenues are worth exploring, such as going beyond least-squares using tools from~\cite{liao2021hessian,loureiro2022fluctuations}, characterizing how quickly our asymptotic analysis kicks in using tools from~\cite{bai2008clt}, looking at more general random projection matrices~\cite{lejeune2022asymptotics}, or relating it to the related sketching procedures that perform linear regression on $Ty \in \rb^{ m}$ and $TX \in \rb^{m \times d}$, where the random matrix $T \in \rb^{m \times n}$ now acts on the left of the design matrix rather than on the right, leading to a form of downsampling often referred to as sketching~\cite{dobriban2019asymptotics,raskutti2016statistical}; see~\cite{chen2023sketched} for a recent  work in this direction.

 \subsubsection*{Acknowledgements}
The author thanks Daniel LeJeune, Andrea Montanari, Bruno Loureiro, Ryan Tibshirani, and Florent Krzakala for feedback on the first version of the manuscript. He also  acknowledges support from the French government under the management of the Agence Nationale de la Recherche as part of the ``Investissements d’avenir'' program, reference
ANR-19-P3IA-0001 (PRAIRIE 3IA Institute), as well as from the European Research Council
(grant SEQUOIA 724063).

 \appendix

\section{Random matrix theory results}
In this appendix, we provide a sketch of proof for classical random matrix theory results presented in \mysec{rmt_inter} and \mysec{implicitreg}, with a proof for the new results from \mysec{spectrace}. For more details, see~\cite{silverstein1995empirical,bai2010spectral}.

 \subsection{Self-consistency equation}
\label{app:spectral}

We follow the proof of~\cite{silverstein1995strong} and derive it in three steps.

\paragraph{First step}
We consider $n \hS = X^\top X = \sum_{i=1}^n x_i x_i^\top$, for $x_i \in \rb^d$ sampled with covariance matrix $\Sigma$ (but not necessarily Gaussian) and write, using the matrix inversion lemma:
\BEAS
 \tr \big[  X^\top X   ( X^\top X  - n z \idm)^{-1} \big] 
& = & \sum_{i=1}^n \tr \Big[  x_i x_i^\top   \Big(   \sum_{j \neq i}  x_j x_j^\top    - n z    \idm +  x_i x_i^\top    \Big)^{-1} \Big]
\\[-.2cm]
& = & \sum_{i=1}^n \frac{  x_i^\top   \big(   \sum_{j \neq i}   x_j x_j^\top  - n z    \idm \big)^{-1}   x_i }{1+ x_i^\top   \big(   \sum_{j \neq i} x_j x_j^\top   - n z    \idm \big)^{-1}  x_i} \\[-.2cm]
& = & n - \sum_{i=1}^n \frac{ 1 }{1+ x_i^\top   \big(   \sum_{j \neq i} x_j x_j^\top  - n z    \idm \big)^{-1}   x_i}.
\EEAS
Together with  $\tr \big[  X^\top X   ( X^\top X  - n z \idm)^{-1} \big] 
= \tr \big[   ( X     X^\top - nz \idm + nz \idm) ( X     X^\top- n z \idm)^{-1} \big]  = 
 n + n z \hphi(z)$, this leads to  the identity
\BEQ
\label{eq:idphi}
 - z \hphi(z) = \frac{1}{n}\sum_{i=1}^n \frac{ 1 }{1+ x_i^\top   \big(   \sum_{j \neq i} x_j x_j^\top  - n z    \idm \big)^{-1} x_i}.
 \EEQ

We also have more generally:
\BEQ
\label{eq:proofA} \hS   (  \hS\!   -\!   z \idm)^{-1}\!
=   \sum_{i=1}^n    x_i x_i^\top     \Big(   \sum_{j \neq i}  x_j x_j^\top    \!-\! n z    \idm +  x_i x_i^\top   \Big)^{-1}  
\!\!=\!\sum_{i=1}^n   \frac{ x_i x_i^\top  \big(   \sum_{j \neq i}   x_j x_j^\top    \!- \!n z    \idm \big)^{-1} }{1+ x_i^\top    \big(   \sum_{j \neq i}   x_j x_j^\top  \!-\! n z    \idm \big)^{-1} x_i}.\!\!\!
\EEQ
 
\paragraph{Second step}

We have, owing to \eq{idphi}, with the notation $\hS_{-i} = \frac{1}{n} \sum_{j \neq i} x_j x_j^\top$ for $i \in \{1,\dots,n\}$, and using $A^{-1} - B^{-1} = - B^{-1} ( A - B) B^{-1}$:
\BEAS
\big( \hS \!-\!  z\idm \big)^{-1}\! -\! 
\big( \!-\! z \hphi(z) \S\! -\! z \idm \big)^{-1}
& = & 
\big( z \hphi(z) \S  + z \idm \big)^{-1}
\Big( \hS  - ( - z \hphi(z) \S ) \Big) \big( \hS -  z\idm \big)^{-1}
\\[-.1cm]
& = & 
\big( z \hphi(z) \S  + z \idm \big)^{-1}
\Big( \frac{1}{n} \sum_{i=1}^n x_i x_i^\top  - ( - z \hphi(z) \S ) \Big) \big( \hS -  z\idm \big)^{-1}
\\[-.1cm]
& = & 
  \big( z \hphi(z) \S  \!+\! z \idm \big)^{-1} \sum_{i=1}^n
  \frac{  x_i x_i^\top   \big(   \sum_{j \neq i} x_j x_j^\top \! -\! n z    \idm \big)^{-1}
\! \! -\! \S ( n \hS \!-\! nz \idm)^{-1} }{1\!+\! x_i^\top   \big(   \sum_{j \neq i} x_j x_j^\top \! -\! n z    \idm \big)^{-1} x_i}
\\[-.1cm]
& = & 
  \big( z \hphi(z) \S  + z \idm \big)^{-1} \frac{1}{n}\sum_{i=1}^n
  \frac{  x_i x_i^\top   \big(       \hS_{-i}   -   z    \idm \big)^{-1}
  - \S (   \hS -  z \idm)^{-1} }{1+ x_i^\top   \big(   n \hS_{-i}  - n z    \idm \big)^{-1} x_i}.
  \EEAS
  We thus get
  \BEA
 \label{eq:Delta} \big( \hS -  z\idm \big)^{-1} 
  & = & \big( - z \hphi(z) \S - z \idm \big)^{-1} \Big( \idm - 
  \frac{1}{n} \sum_{i=1}^n
  \frac{  x_i x_i^\top   \big(       \hS_{-i}   -   z    \idm \big)^{-1}
  - \S (   \hS -  z \idm)^{-1} }{1+ x_i^\top   \big(   n \hS_{-i}  - n z    \idm \big)^{-1} x_i}\Big)
  \\[-.1cm]
 \notag & = &   \big( - z \hphi(z) \S - z \idm \big)^{-1}  ( \idm - \Delta).
  \EEA

The main property we will leverage is that $\Delta $ will almost certainly be  ``negligible''.
For this, we need that 
$\ds \tr\big[ \big( - z \hphi(z) \S - z \idm \big)^{-1} \Delta \big] = o(d)$, and we simply need to study each of the $n$ terms, and show that they are $o(d)$. The key is that $\big(    \hS_{-i}  -  z    \idm \big)^{-1}$ is independent of~$x_i$ and that for any deterministic (or independent random bounded) matrix, 
$ \tr [ ( z_iz_i^\top - \idm) N ] $ is small enough with a strong probabilistic control~\cite[Lemma 3.1]{silverstein1995empirical}. This is where we need i.i.d.~components for $z_i$ with sufficient moments (we assumed sub-Gaussian for simplicity, but weaker assumptions could be used to obtain the same almost-sure result). We can for example rely the on Hanson-Wright inequality~\cite{hansonwright}, which leads to, for a constant $c>0$:
$$
\P\Big[  \big| z_i^\top N z_i - \tr (N) \big| \leqslant c \big(  t \| N \|_{\rm op} + \sqrt{t} \| N \|_F \big)\Big] \geqslant 1 - 2 e^{-t}.
$$
This is then applied to $N$ dominated by $\Sigma$, and thus $\| N\|_F = O( \| \Sigma\|_F) = O(\sqrt{d}) = o(d)$, which is sufficient for the asymptotic result and hints at a rate in $O(1/\sqrt{d})$~\cite{bai2008clt}. See \cite{silverstein1995strong} for a detailed proof.

Overall, once we can neglect the term in $\Delta$,  we  get:
$
\tr\big[  ( -    z \hphi(z) \S -   z \idm)^{-1} \big] \sim 
 \tr \big[  ( \hS    -   z \idm)^{-1} \big] ,
$
and thus
\BEQ
\label{eq:proofB}
 \tr \big[ \big(  \hS -   z \idm \big)^{-1} \big] 
\sim \frac{- 1}{z\hphi(z)}  \tr\Big[  \Big(  \S +  \frac{1}{\hphi(z)} \idm \Big)^{-1} \Big]
= \frac{- d}{z } +  \frac{1}{z} \tr\Big[  \S \Big(  \S +  \frac{1}{\hphi(z)} \idm \Big)^{-1} \Big].
\EEQ

\paragraph{Third step}
We can rewrite
\BEA
\notag
 \tr \big[ \big(  \hS   -   z \idm \big)^{-1} \big]
& = &  \frac{1}{z}  \tr \big[ \big( z \idm - \hS      + \hS   \big) \big(  \hS    -   z \idm \big)^{-1} \big] \\
\notag & = &  - \frac{d}{z} 
+ \frac{1}{z}  \tr \big[   \hS    \big(  \hS    -   z \idm \big)^{-1} \big] =  - \frac{d}{z} 
+ \frac{1}{z}  \tr \big[    X   X^\top  \big(  X   X^\top-   nz \idm \big)^{-1} \big] \\
\label{eq:proofC} & = &  - \frac{d}{z} 
+ \frac{1}{z}  \tr \big[    \big( X   X^\top -nz \idm + nz \idm\big)  \big(  X  X^\top-   nz \idm \big)^{-1} \big]
= \frac{n-d}{z} 
+ n \hphi(z).
 \EEA
 Following~\cite{silverstein1995strong} and combining \eq{proofB} and \eq{proofC}, this leads to $\hphi(z) \to  {\varphi}(z)$, with 
 \BEQ
\label{eq:proofE}  \varphi(z)  + \frac{1}{z} = \frac{1}{nz} \tr\Big[  \S \Big(  \S +  \frac{1}{\varphi(z)} \idm \Big)^{-1} \Big],
  \EEQ
  which is the desired self-consistent equation in \eq{varphilambda} in \mysec{implicitreg}.
  
And even more intuitively, since $\tr \big[ \hS ( \hS - z \idm)^{-1} \big]
= d + z \tr \big[  ( \hS - z \idm)^{-1} \big] = nz ( \hphi(z)+ \frac{1}{z})$, we get \eq{eqdf} from \mysec{rmt_inter}:
 \BEQ
\label{eq:proofF} \tr\big[  \hS \big(   \hS - z\idm \big)^{-1} \big] 
 \sim\tr\Big[  \S \Big(  \S +  \frac{1}{\varphi(z)} \idm \Big)^{-1} \Big].
 \EEQ
 We have for $z = -\lambda$, with $\lambda>0$:
  $\tr\big[  \hS \big(   \hS + \lambda \idm \big)^{-1} \big] 
 \sim\tr\Big[  \S \Big(  \S +  \frac{1}{\varphi(-\lambda)} \idm \Big)^{-1} \Big] ,$ and thus
$$
 \varphi(-\lambda)  - \frac{1}{\lambda} = - \frac{1}{n\lambda} \tr\Big[  \S \Big(  \S +  \frac{1}{\varphi(-\lambda)} \idm \Big)^{-1} \Big] = - \frac{1}{n\lambda}  {\rm df}_1\Big(
 \frac{1}{\varphi(-\lambda)}
 \Big),
 $$
 leading to $\ds\lambda \varphi(-\lambda) = 1 - \frac{1}{n }  {\rm df}_1\Big(
 \frac{1}{\varphi(-\lambda)}
 \Big),$
 and thus the desired inequality with $\kappa(\lambda) = \frac{1}{\varphi(-\lambda)}$, presented in \mysec{implicitreg}.

\subsection{Equivalents of spectral functions}
\label{app:spectraltrace}

In this section, we prove Prop.~\ref{prop:spectral} and  Prop.~\ref{prop:spectralK}. Following~\cite{dobriban2018high}, we start with an asymptotic equivalent based on differentiation (see formal justification in~\cite{dobriban2018high}).
See~\cite{dar2021common,lejeune2022asymptotics} for similar results based more strongly on differentiation (which is only used here to derive an equivalent for $ \tr\big[   \big(   \hS - z\idm \big)^{-2} \big] $).

\paragraph{Using differentiation}
We have, by differentiating \eq{proofE} with respect to $z$:
\BEAS
\varphi(z) + z \varphi'(z)
& = & \frac{1}{n} \tr\Big[  \S \Big(  \S +  \frac{1}{\varphi(z)} \idm \Big)^{-2} \Big] \frac{ \varphi'(z)}{\varphi(z)^2},
\EEAS
which leads to
$\ds 
\frac{\varphi(z)}{\varphi'(z)}
= \frac{1}{n} \tr\Big[  \S \Big(  \S +  \frac{1}{\varphi(z)} \idm \Big)^{-2} \Big] \frac{ 1}{\varphi(z)^2} - z.
$
Thus,  differentiating \eq{proofF} with respect to $z$ and using the bound on $\frac{\varphi(z)}{\varphi'(z)}$ above, we get:
\BEAS
\tr\big[  \hS \big(   \hS - z\idm \big)^{-2} \big] 
&\! \sim \! &  \tr\Big[  \S \Big(  \S +  \frac{1}{\varphi(z)} \idm \Big)^{-2} \Big] \frac{ \varphi'(z)}{\varphi(z)^2}
=   \frac{  n \tr\big[  \S \big(  \S +  \frac{1}{\varphi(z)} \idm \big)^{-2} \big]}
{    \tr\big[  \S \big(  \S +  \frac{1}{\varphi(z)} \idm \big)^{-2} \big] \frac{ 1}{\varphi(z)}  - nz \varphi(z)   } 
\\
&\! = \!& \frac{ n \tr\big[  \S \big(  \S +  \frac{1}{\varphi(z)} \idm \big)^{-2} \big]}
{  \tr\big[  \S \big(  \S +  \frac{1}{\varphi(z)} \idm \big)^{-2} \big] \frac{ 1}{\varphi(z)}
 + n - \tr\big[  \S \big(  \S +  \frac{1}{\varphi(z)} \idm \big)^{-1} \big]   } = \frac{ n \tr\big[  \S \big(  \S +  \frac{1}{\varphi(z)} \idm \big)^{-2} \big]}
{    n -  \tr\big[  \S^2 \big(  \S +  \frac{1}{\varphi(z)} \idm \big)^{-2} \big] } .
\EEAS
This leads to the asymptotic equivalent
\BEA
\notag \tr\big[   \big(   \hS - z\idm \big)^{-2} \big] 
& = & \frac{1}{z} \tr\big[  (z \idm - \hS + \hS) \big(   \hS - z\idm \big)^{-2} \big] 
= \frac{1}{z} \tr\big[   \hS  \big(   \hS - z\idm \big)^{-2} \big]  
-  \frac{1}{z} \tr\big[  \big(   \hS - z\idm \big)^{-1} \big]  \\
\label{eq:prdr}
& \sim & \frac{1}{z} 
\frac{   n\tr\big[  \S \big(  \S +  \frac{1}{\varphi(z)} \idm \big)^{-2} \big]}
{    n -  \tr\big[  \S^2 \big(  \S +  \frac{1}{\varphi(z)} \idm \big)^{-2} \big]} - \frac{1}{z} 
\tr \big[   \big( - z \varphi(z) \S - z \idm \big)^{-1} \big],
\EEA
which we will need later.

\paragraph{Proof of \eq{trA1} and \eq{trA2}}
We now first show
  \BEAS
\tr \big[ A \big( \hS -  z\idm \big)^{-1} \big] 
& \sim & \tr \big[ A \big( - z \hphi(z) \S - z \idm \big)^{-1} ]
\sim   \tr \big[ A \big( - z \varphi(z) \S - z \idm \big)^{-1} ],
\\
& = &    \frac{-1}{z \varphi(z)} \tr \Big[ A \Big(  \S + \frac{1}{\varphi(z)} \idm \Big)^{-1}\Big],
\EEAS
where the last quantity is equivalent to  $  -\frac{d}{z} \int_0^{+\infty} \frac{  d\nu_A(\sigma)}{1 + \sigma \varphi(z)}$.
We have,  using \eq{Delta}:
$$
  \tr \big[ A \big( \hS -  z\idm \big)^{-1} \big] 
-  \tr \big[ A \big( - z \hphi(z) \S - z \idm \big)^{-1} ] 
= - \tr \big[ A   \big( - z \hphi(z) \S - z \idm \big)^{-1} \Delta \big],
$$
  which is negligible as soon as $\| A\|_{\rm op}$ is bounded (using the same arguments as in Appendix~\ref{app:spectral}). We can then express 
  $\tr \big[ A \big( \S +\frac{1}{\varphi(z)} \idm \big)^{-1} ]$ as $  \int_0^{+\infty} \frac{ d\nu_A(\sigma)}{\sigma + \frac{1}{\varphi(z)} }$. This 
   leads to the desired result in Prop.~\ref{prop:spectral}.
  
  \paragraph{Proof of \eq{trAB1} and \eq{trAB2}}

  For the quadratic form, we have for any matrices $A$ and $B$, still using \eq{Delta}:
  \BEAS
  & & \tr \big[ A  \big( \hS -  z\idm \big)^{-1}  B  \big( \hS -  z\idm \big)^{-1}  \big] 
  \\
  &\!\! \!\!\!= \!& 
  \tr \big[ A   \big( - z \hphi(z) \S - z \idm \big)^{-1}  ( \idm - \Delta)  B \big( - z \hphi(z) \S - z \idm \big)^{-1}    
   ( \idm - \Delta)  \big] 
\\
 &\!\! \!\!\!= \!& 
  \tr \big[ A   \big( - z \hphi(z) \S - z \idm \big)^{-1}    B   \big( - z \hphi(z) \S - z \idm \big)^{-1}    \big] 
  + \tr \big[ A   \big( - z \hphi(z) \S - z \idm \big)^{-1}  \Delta   B     \big( - z \hphi(z) \S - z \idm \big)^{-1}   \Delta \big] \\
  & & -   \tr \big[ A   \big( - z \hphi(z) \S - z \idm \big)^{-1} \Delta  B \big( - z \hphi(z) \S - z \idm \big)^{-1} 
\! -    \tr \big[ A   \big( - z \hphi(z) \S - z \idm \big)^{-1}   B \big( - z \hphi(z) \S - z \idm \big)^{-1}   \Delta  
\big] .
  \EEAS
  The last two terms are negligible with the same arguments as in Appendix~\ref{app:spectral} as soon as $\|A\|_{\rm op}$ and $\| B\|_{\rm op}$ are bounded. 
  We have, for the second term:
  \BEAS
  & & \tr \big[ A   \big( - z \hphi(z) \S - z \idm \big)^{-1}  \Delta   B       \big( - z \hphi(z) \S - z \idm \big)^{-1}\Delta    \big] \\
  & \!\!\!\!= \!\!& 
  \frac{1}{n^2}   \sum_{i,j=1}^n
  \tr \Big[  A   \big( \!-\! z \hphi(z) \S \!-\! z \idm \big)^{-1} 
    \frac{  ( x_i x_i^\top - \S)   \big(       \hS_{-i}  \! -\!   z    \idm \big)^{-1}
  }{1+ x_i^\top   \big(    n  \hS_{-i}  \!- \!  nz    \idm \big)^{-1} x_i}
   B       \big(\! -\! z \hphi(z) \S\! -\! z \idm \big)^{-1}
    \frac{  ( x_j x_j^\top - \S)   \big(       \hS_{-j}  \! - \!  z    \idm \big)^{-1}
  }{1+ x_j^\top   \big(    n  \hS_{-j}\!  - \!  nz    \idm \big)^{-1} x_j}\Big]
  \\
   & \!\!\!\!= \!\!& 
    \frac{1}{n^2}   \sum_{i,j=1}^n \frac{\tr \big[
    A   \big( \!-\! z \hphi(z) \S \!-\! z \idm \big)^{-1} 
    ( x_i x_i^\top - \S)   \big(       \hS_{-i}  \! - \!  z    \idm \big)^{-1}
     B       \big( \!-\! z \hphi(z) \S \!-\! z \idm \big)^{-1}
     ( x_j x_j^\top - \S)   \big(       \hS_{-j}  \! -\!   z    \idm \big)^{-1}
\big]
   }{\big(
1+ x_i^\top   \big(    n  \hS_{-i}  \!-\!   nz    \idm \big)^{-1} x_i
    \big)\big(1+ x_j^\top   \big(    n  \hS_{-j}  \!-\!   nz    \idm \big)^{-1} x_j\big)
    } .
    \EEAS
    When $i \neq j$, then we can separate terms with $x_i x_i^\top - \S$ and $x_j x_j^\top - \S$, which end up being negligible, thus leading to an equivalent
\BEAS
   \!\! \frac{1}{n^2}  \! \sum_{i=1}^n \!\frac{\tr \big[
    A   \big( \!- \!z \hphi(z) \S \!-\! z \idm \big)^{-1} 
    ( x_i x_i^\top\! -\! \S)   \big(       \hS_{-i} \!  -\!   z    \idm \big)^{-1}
     \!B       \big( \!-\! z \hphi(z) \S \!- \!z \idm \big)^{-1}
     ( x_i x_i^\top \!-\! \S)   \big(       \hS_{-i}  \! - \!  z    \idm \big)^{-1}
\big]
   }{\big(
1+ x_i^\top   \big(    n  \hS_{-i}  \!- \!  nz    \idm \big)^{-1} x_i
    \big)^2
    } .  \EEAS
    To study its asymptotic limit, we need to characterize the asymptotic equivalent of the quantity
   $   \tr \big[ C ( x_i x_i^\top - \S)   D  ( x_i x_i^\top - \S)  \big]  
   =    \tr \big[ \S^{1/2} C \S^{1/2} ( z_i z_i^\top - \idm)  \S^{1/2} D \S^{1/2}  ( z_i z_i^\top - \idm)  \big] $, with $C$ and $D$ bounded in operator norm. For $M= \S^{1/2} C \S^{1/2} $, and $N = \S^{1/2} D \S^{1/2} $,  we can write:
   \BEAS
   \tr \big[ M (z_i z_i^\top\! -\! \idm) N(z_i z_i^\top\! -\! \idm) \big]
   \!-\!  \tr(M) \tr(N) 
   & \!= \!&  (z_i^\top M z_i \!-\! \tr(M) ) 
   (z_i^\top N z_i\! -\! \tr(N) ) \\
   & & + \tr(M)   (z_i^\top N z_i \!-\! \tr(N) )  + \tr(N)   (z_i^\top M z_i \!-\! \tr(M) )   \\
   & &  - \tr[ (MN+NM) (z_i z_i^\top\! - \! \idm) ] \\
   & \!= \!& O_p( \| M \|_F \cdot \|N\|_F \!+\! \tr(M) \| N\|_F \!+\! \tr(N) \| M\|_F \!+\! \| NM\|_F).
   \EEAS
   using the property from Appendix~\ref{app:spectral} obtain from the i.i.d.~assumption on the components of $z_i$, which is negligible compared to the term $ \tr(M) \tr(N) $. Thus, using in addition that $ \hS_{-j}$ is asymptotically equivalent to $\hS$, we get the equivalent
   $$
    \frac{1}{n^2}   \sum_{i=1}^n \frac{\tr \big[
     ( \hS - z \idm)^{-1} A   \big( - z \hphi(z) \S - z \idm \big)^{-1} \Sigma \big]
     \cdot 
     \tr \big[
     ( \hS - z \idm)^{-1} B   \big( - z \hphi(z) \S - z \idm \big)^{-1} \Sigma \big]
      }{\big(
1+ x_i^\top   \big(    n  \hS_{-i}  -   nz    \idm \big)^{-1} x_i
    \big)^2}.
$$
   We thus overall have
  \BEAS
  & & \tr \big[ A   \big( - z \hphi(z) \S - z \idm \big)^{-1}  \Delta   B   \Delta^\top   \big( - z \hphi(z) \S - z \idm \big)^{-1}    \big] \\
  & \sim & 
     \tr \big[
     ( \hS - z \idm)^{-1} A   \big( - z \hphi(z) \S - z \idm \big)^{-1} \Sigma \big]
     \cdot 
     \tr \big[
     ( \hS - z \idm)^{-1} B   \big( - z \hphi(z) \S - z \idm \big)^{-1} \Sigma \big]
    \cdot \square
\\
 & \sim & 
     \tr \big[
    A   \big( - z \varphi(z) \S - z \idm \big)^{-2} \Sigma \big]
     \cdot 
     \tr \big[
    B   \big( - z \varphi(z) \S - z \idm \big)^{-2} \Sigma \big]
      \cdot \square
\EEAS
with $\ds \square =  \frac{1}{n^2}   \sum_{i=1}^n \frac{1  }{\big(
1+ x_i^\top   \big(    n  \hS_{-i}  -   nz    \idm \big)^{-1} x_i
    \big)^2
    }  $.
This leads to:
 \BEAS
    \tr \big[ A  \big( \hS -  z\idm \big)^{-1}  B  \big( \hS -  z\idm \big)^{-1}  \big]  
 & \sim &  \tr \big[ A   \big( - z \varphi(z) \S - z \idm \big)^{-1}    B   \big( - z \varphi(z) \S - z \idm \big)^{-1}    \big] 
 \\
 & & + \tr \big[
    A   \big( - z \varphi(z) \S - z \idm \big)^{-2} \Sigma \big]
     \cdot 
     \tr \big[
    B   \big( - z \varphi(z) \S - z \idm \big)^{-2} \Sigma \big]
    \cdot   \square.
 \EEAS
 To obtain an equivalent of $\square$, we consider the case
 $A= B = \idm$, to get:
 \BEAS
 \tr \big[    \big( \hS -  z\idm \big)^{-2}  \big]
&  \sim & 
 \tr \big[    \big( - z \varphi(z) \S - z \idm \big)^{-2}      \big] 
 + \big( \tr \big[
        \big( - z \varphi(z) \S - z \idm \big)^{-2} \Sigma \big]
     \big)^2
      \cdot \square,
\EEAS 
which allows to compute an equivalent of $\square$, as, using \eq{prdr}, with $z \varphi(z) \sim  {\rm df}_1(1/\varphi(z)) - \frac{1}{n}$.
\BEAS
\square 
& \sim & 
\frac{ \tr \big[    \big( \hS -  z\idm \big)^{-2}  \big] - \tr \big[    \big( - z \varphi(z) \S - z \idm \big)^{-2}      \big]  }{
\big( \tr \big[
        \big( - z \varphi(z) \S - z \idm \big)^{-2} \Sigma \big]
     \big)^2} \\
     & \sim & 
\frac{ \frac{1}{z} 
\frac{   n\tr\big[  \S \big(  \S +  \frac{1}{\varphi(z)} \idm \big)^{-2} \big]}
{    n -  {\rm df}_2(1/\varphi(z))  } - \frac{1}{z} 
\tr \big[   \big( - z \varphi(z) \S - z \idm \big)^{-1} ] - \tr \big[    \big( - z \varphi(z) \S - z \idm \big)^{-2}      \big]  }{
\big( \tr \big[
        \big( - z \varphi(z) \S - z \idm \big)^{-2} \Sigma \big]
     \big)^2}
\\
     & \sim & 
\frac{ \frac{1}{z} 
\frac{   n\tr\big[  \S \big(  \S +  \frac{1}{\varphi(z)} \idm \big)^{-2} \big]}
{    n -  {\rm df}_2(1/\varphi(z))  } + \frac{1}{z ^2 \varphi(z)} 
\tr \big[   \big( \S + \frac{1}{\varphi(z)}\idm \big)^{-1} ] -
\frac{1}{z ^2 \varphi(z) }   \tr \big[    \frac{1}{\varphi(z)}\big(   \S + \frac{1}{\varphi(z)}\idm \big)^{-2}      \big]  }{
 \big( 
 \frac{1}{z^2 \varphi(z)}\tr \big[ \frac{1}{\varphi(z)} 
        \big(  \S + \frac{1}{\varphi(z)} \idm \big)^{-2} \Sigma \big]
     \big)^2}
\\
     & \sim & 
\frac{ \frac{1}{z} 
\frac{   n\tr\big[  \S \big(  \S +  \frac{1}{\varphi(z)} \idm \big)^{-2} \big]}
{    n -  {\rm df}_2(1/\varphi(z))  }  + 
\frac{1}{z ^2 \varphi(z) }   \tr \big[   \S \big(   \S + \frac{1}{\varphi(z)}\idm \big)^{-2}      \big]  }{
 \big( 
 \frac{1}{z^2 \varphi(z)}\tr \big[ \frac{1}{\varphi(z)} 
        \big(  \S + \frac{1}{\varphi(z)} \idm \big)^{-2} \Sigma \big]
     \big)^2}
  \sim  
\frac{ \frac{1}{z} 
\frac{   n }
{    n -  {\rm df}_2(1/\varphi(z))  }  + 
\frac{1}{z ^2 \varphi(z) }     }{
 \big( 
 \frac{1}{z^2 \varphi(z)} \big)^2 \frac{1}{\varphi(z)} \tr \big[ \frac{1}{\varphi(z)} 
        \big(  \S + \frac{1}{\varphi(z)} \idm \big)^{-2} \Sigma \big]
     }
\\
   & \sim & 
\frac{ 
\frac{   n z \varphi(z) }
{    n -  {\rm df}_2(1/\varphi(z))  }  + 
1      }{
 \frac{1}{z^2 \varphi(z)^2}  \tr \big[ \frac{1}{\varphi(z)} 
        \big(  \S + \frac{1}{\varphi(z)} \idm \big)^{-2} \Sigma \big]
     }
=
\frac{ 
\frac{  {\rm df}_1(1/\varphi(z))   - n }
{    n -  {\rm df}_2(1/\varphi(z))  }  + 
1      }{
 \frac{1}{z^2 \varphi(z)^2}  \tr \big[ \frac{1}{\varphi(z)} 
        \big(  \S + \frac{1}{\varphi(z)} \idm \big)^{-2} \Sigma \big]
     }
\\
  & \sim & 
\frac{ 
\frac{  {\rm df}_1(1/\varphi(z))   - n }
{    n -  {\rm df}_2(1/\varphi(z))  }  + 
1      }{
 \frac{1}{z^2 \varphi(z)^2}  
 \big(
   {\rm df}_1(1/\varphi(z))  -   {\rm df}_2(1/\varphi(z)) 
 \big)     }
 = \frac{z^2 \varphi(z)^2}{ n -  {\rm df}_2(1/\varphi(z))  }.
\EEAS
 
 Overall, we get
  \BEAS
 & &  \tr \big[ A  \big( \hS -  z\idm \big)^{-1}  B  \big( \hS -  z\idm \big)^{-1}  \big]  \\
 & \sim & \textstyle \frac{1}{z^2 \varphi(z)^2}  \tr \big[ A   \big( \S + \frac{1}{\varphi(z)} \idm \big)^{-1}    B   \big( \S + \frac{1}{\varphi(z)} \idm \big)^{-1}    \big] 
 \\
 & & \textstyle+ \frac{1}{z^4 \varphi(z)^4}   \tr \big[
    A    \big( \S + \frac{1}{\varphi(z)} \idm \big)^{-2}  \Sigma \big]
     \cdot 
     \tr \big[
    B    \big( \S + \frac{1}{\varphi(z)} \idm \big)^{-2}  \Sigma \big]
    \cdot   \frac{z^2 \varphi(z)^2}{ n -  {\rm df}_2(1/\varphi(z))  }
\\
& \sim &\textstyle \frac{1}{z^2 \varphi(z)^2}  \tr \big[ A   \big( \S + \frac{1}{\varphi(z)} \idm \big)^{-1}    B   \big( \S + \frac{1}{\varphi(z)} \idm \big)^{-1}    \big] 
 \\
 & & \textstyle + \frac{1}{z^2 \varphi(z)^2}   \tr \big[
    A    \big( \S + \frac{1}{\varphi(z)} \idm \big)^{-2}  \Sigma \big]
     \cdot 
     \tr \big[
    B    \big( \S + \frac{1}{\varphi(z)} \idm \big)^{-2}  \Sigma \big]
    \cdot   \frac{1}{ n -  {\rm df}_2(1/\varphi(z))  },
 \EEAS
 which is \eq{trAB2}.

We also have, by writing  $ \hS\big( \hS -  z\idm \big)^{-1} = \idm +  z  \big( \hS -  z\idm \big)^{-1}$:
  \BEAS
 & &  \tr \big[ A  \hS\big( \hS -  z\idm \big)^{-1}  B  \big( \hS -  z\idm \big)^{-1}  \big]  \\
 &= & \textstyle z \tr \big[ A  \big( \hS -  z\idm \big)^{-1}  B  \big( \hS -  z\idm \big)^{-1}  \big] 
 +\tr \big[ A   B  \big( \hS -  z\idm \big)^{-1}  \big] \\
& \sim & - \frac{1}{z \varphi(z)}   \tr \big[ A   B  \big( \S  + \frac{1}{\varphi(z)} \idm \big)^{-1}  \big] +  \frac{1}{z  \varphi(z)^2}  \tr \big[ A   \big( \S + \frac{1}{\varphi(z)} \idm \big)^{-1}    B   \big( \S + \frac{1}{\varphi(z)} \idm \big)^{-1}    \big] 
 \\
 & & \textstyle+ \frac{1}{z \varphi(z)^2}   \tr \big[
    A    \big( \S + \frac{1}{\varphi(z)} \idm \big)^{-2}  \Sigma \big]
     \cdot 
     \tr \big[
    B    \big( \S + \frac{1}{\varphi(z)} \idm \big)^{-2}  \Sigma \big]
    \cdot   \frac{1}{ n -  {\rm df}_2(1/\varphi(z))  }
\\
 & \sim &\textstyle - \frac{1}{z \varphi(z)}   \tr \big[ A   B  \big( \S  + \frac{1}{\varphi(z)} \idm \big)^{-1}  \big] +  \frac{1}{z  \varphi(z)}  \tr \big[ A   \frac{1}{\varphi(z)}\big( \S + \frac{1}{\varphi(z)} \idm \big)^{-1}    B   \big( \S + \frac{1}{\varphi(z)} \idm \big)^{-1}    \big] 
 \\
 & & \textstyle+ \frac{1}{z \varphi(z)^2}   \tr \big[
    A    \big( \S + \frac{1}{\varphi(z)} \idm \big)^{-2}  \Sigma \big]
     \cdot 
     \tr \big[
    B    \big( \S + \frac{1}{\varphi(z)} \idm \big)^{-2}  \Sigma \big]
    \cdot   \frac{1}{ n -  {\rm df}_2(1/\varphi(z))  }
\\
 & \sim &\textstyle - \frac{1}{z \varphi(z)}   \tr \big[ A   \S  \big( \S  + \frac{1}{\varphi(z)} \idm \big)^{-1} B  \big( \S  + \frac{1}{\varphi(z)} \idm \big)^{-1}  \big]  
 \\
 & & \textstyle+ \frac{1}{z \varphi(z)^2}   \tr \big[
    A    \big( \S + \frac{1}{\varphi(z)} \idm \big)^{-2}  \Sigma \big]
     \cdot 
     \tr \big[
    B    \big( \S + \frac{1}{\varphi(z)} \idm \big)^{-2}  \Sigma \big]
    \cdot   \frac{1}{ n -  {\rm df}_2(1/\varphi(z))  }.
\EEAS

We also finally have by using again $ \hS\big( \hS -  z\idm \big)^{-1} = \idm +  z  \big( \hS -  z\idm \big)^{-1}$:
  \BEAS
 & &  \tr \big[ A  \hS\big( \hS -  z\idm \big)^{-1}  B \hS\big( \hS -  z\idm \big)^{-1} \big]  \\
& \sim & \textstyle  \tr \big[ A   \S  \big( \S  + \frac{1}{\varphi(z)} \idm \big)^{-1} B  \S\big( \S  + \frac{1}{\varphi(z)} \idm \big)^{-1}  \big]  
 \\
 & &\textstyle + \frac{1}{ \varphi(z)^2}   \tr \big[
    A    \big( \S + \frac{1}{\varphi(z)} \idm \big)^{-2}  \Sigma \big]
     \cdot 
     \tr \big[
    B    \big( \S + \frac{1}{\varphi(z)} \idm \big)^{-2}  \Sigma \big]
    \cdot   \frac{1}{ n -  {\rm df}_2(1/\varphi(z))  },
\EEAS
which is \eq{trAB1}.

     \bibliography{double_descent}

\begin{thebibliography}{10}

\bibitem{bai2008clt}
Zhidong Bai and Jack~W. Silverstein.
\newblock {CLT} for linear spectral statistics of large-dimensional sample
  covariance matrices.
\newblock In {\em Advances In Statistics}, pages 281--333. World Scientific,
  2008.

\bibitem{bai2010spectral}
Zhidong Bai and Jack~W. Silverstein.
\newblock {\em Spectral Analysis of Large Dimensional Random Matrices},
  volume~20.
\newblock Springer, 2010.

\bibitem{bai2008limit}
Zhidong Bai and Yong-Qua Yin.
\newblock Limit of the smallest eigenvalue of a large dimensional sample
  covariance matrix.
\newblock In {\em Advances In Statistics}, pages 108--127. World Scientific,
  2008.

\bibitem{bartlett2020benign}
Peter~L. Bartlett, Philip~M. Long, G{\'a}bor Lugosi, and Alexander Tsigler.
\newblock Benign overfitting in linear regression.
\newblock {\em Proceedings of the National Academy of Sciences},
  117(48):30063--30070, 2020.

\bibitem{bartlett_montanari_rakhlin_2021}
Peter~L. Bartlett, Andrea Montanari, and Alexander Rakhlin.
\newblock Deep learning: a statistical viewpoint.
\newblock {\em Acta Numerica}, 30:87–201, 2021.

\bibitem{belkin2019reconciling}
Mikhail Belkin, Daniel Hsu, Siyuan Ma, and Soumik Mandal.
\newblock Reconciling modern machine-learning practice and the classical
  bias--variance trade-off.
\newblock {\em Proceedings of the National Academy of Sciences},
  116(32):15849--15854, 2019.

\bibitem{belkin2020two}
Mikhail Belkin, Daniel Hsu, and Ji~Xu.
\newblock Two models of double descent for weak features.
\newblock {\em SIAM Journal on Mathematics of Data Science}, 2(4):1167--1180,
  2020.

\bibitem{FCM:Caponetto+Vito:2007}
Andrea Caponnetto and Ernesto {de Vito}.
\newblock Optimal rates for regularized least-squares algorithm.
\newblock {\em Foundations of Computational Mathematics}, 7(3):331--368, 2007.

\bibitem{chen2023sketched}
Xin Chen, Yicheng Zeng, Siyue Yang, and Qiang Sun.
\newblock Sketched ridgeless linear regression: The role of downsampling.
\newblock Technical Report 2302.01088, arXiv, 2023.

\bibitem{cheng2022dimension}
Chen Cheng and Andrea Montanari.
\newblock Dimension free ridge regression.
\newblock Technical Report 2210.08571, arXiv, 2022.

\bibitem{cui2021generalization}
Hugo Cui, Bruno Loureiro, Florent Krzakala, and Lenka Zdeborov{\'a}.
\newblock Generalization error rates in kernel regression: The crossover from
  the noiseless to noisy regime.
\newblock {\em Advances in Neural Information Processing Systems},
  34:10131--10143, 2021.

\bibitem{dar2021common}
Yehuda Dar, Daniel LeJeune, and Richard~G. Baraniuk.
\newblock The common intuition to transfer learning can win or lose: Case
  studies for linear regression.
\newblock Technical Report 2103.05621, arXiv, 2021.

\bibitem{dobriban2019asymptotics}
Edgar Dobriban and Sifan Liu.
\newblock Asymptotics for sketching in least squares regression.
\newblock {\em Advances in Neural Information Processing Systems}, 32, 2019.

\bibitem{dobriban2018high}
Edgar Dobriban and Stefan Wager.
\newblock High-dimensional asymptotics of prediction: Ridge regression and
  classification.
\newblock {\em The Annals of Statistics}, 46(1):247--279, 2018.

\bibitem{geiger2019scaling}
Mario Geiger, Arthur Jacot, Stefano Spigler, Franck Gabriel, Levent Sagun,
  St{\'e}phane d’Ascoli, Giulio Biroli, Cl{\'e}ment Hongler, and Matthieu
  Wyart.
\newblock Scaling description of generalization with number of parameters in
  deep learning.
\newblock {\em Journal of Statistical Mechanics: Theory and Experiment},
  (2):023401, 2020.

\bibitem{gunasekar2018characterizing}
Suriya Gunasekar, Jason Lee, Daniel Soudry, and Nathan Srebro.
\newblock Characterizing implicit bias in terms of optimization geometry.
\newblock In {\em International Conference on Machine Learning}, 2018.

\bibitem{hastie2019surprises}
Trevor Hastie, Andrea Montanari, Saharon Rosset, and Ryan~J. Tibshirani.
\newblock Surprises in high-dimensional ridgeless least squares interpolation.
\newblock {\em The Annals of Statistics}, 50(2):949--986, 2022.

\bibitem{hastie_GAM}
Trevor~J. Hastie and Robert~J. Tibshirani.
\newblock {\em Generalized Additive Models}.
\newblock Chapman \& Hall, 1990.

\bibitem{hoerl1970ridge}
Arthur~E. Hoerl and Robert~W. Kennard.
\newblock Ridge regression: Biased estimation for nonorthogonal problems.
\newblock {\em Technometrics}, 12(1):55--67, 1970.

\bibitem{hsu2012random}
Daniel Hsu, Sham~M. Kakade, and Tong Zhang.
\newblock Random design analysis of ridge regression.
\newblock In {\em Conference on Learning Theory}, 2012.

\bibitem{jacot2020implicit}
Arthur Jacot, Berfin Simsek, Francesco Spadaro, Cl{\'e}ment Hongler, and Franck
  Gabriel.
\newblock Implicit regularization of random feature models.
\newblock In {\em International Conference on Machine Learning}, 2020.

\bibitem{ledoit2011eigenvectors}
Olivier Ledoit and Sandrine P{\'e}ch{\'e}.
\newblock Eigenvectors of some large sample covariance matrix ensembles.
\newblock {\em Probability Theory and Related Fields}, 151(1-2):233--264, 2011.

\bibitem{lejeune2022asymptotics}
Daniel LeJeune, Pratik Patil, Hamid Javadi, Richard~G. Baraniuk, and Ryan~J.
  Tibshirani.
\newblock Asymptotics of the sketched pseudoinverse.
\newblock Technical Report 2211.03751, arXiv, 2022.

\bibitem{liao2020random}
Zhenyu Liao, Romain Couillet, and Michael~W. Mahoney.
\newblock A random matrix analysis of random fourier features: beyond the
  gaussian kernel, a precise phase transition, and the corresponding double
  descent.
\newblock {\em Advances in Neural Information Processing Systems}, 33, 2020.

\bibitem{liao2021hessian}
Zhenyu Liao and Michael~W. Mahoney.
\newblock Hessian eigenspectra of more realistic nonlinear models.
\newblock {\em Advances in Neural Information Processing Systems}, 34, 2021.

\bibitem{loureiro2022fluctuations}
Bruno Loureiro, C{\'e}dric Gerbelot, Maria Refinetti, Gabriele Sicuro, and
  Florent Krzakala.
\newblock Fluctuations, bias, variance \& ensemble of learners: Exact
  asymptotics for convex losses in high-dimension.
\newblock In {\em International Conference on Machine Learning}, 2022.

\bibitem{mei2019generalization}
Song Mei and Andrea Montanari.
\newblock The generalization error of random features regression: Precise
  asymptotics and the double descent curve.
\newblock {\em Communications on Pure and Applied Mathematics}, 75(4):667--766,
  2022.

\bibitem{montanari2022interpolation}
Andrea Montanari and Yiqiao Zhong.
\newblock The interpolation phase transition in neural networks: Memorization
  and generalization under lazy training.
\newblock {\em The Annals of Statistics}, 50(5):2816--2847, 2022.

\bibitem{mourtada2022elementary}
Jaouad Mourtada and Lorenzo Rosasco.
\newblock An elementary analysis of ridge regression with random design.
\newblock {\em Comptes Rendus. Math{\'e}matique}, 360(G9):1055--1063, 2022.

\bibitem{raskutti2016statistical}
Garvesh Raskutti and Michael~W. Mahoney.
\newblock A statistical perspective on randomized sketching for ordinary
  least-squares.
\newblock {\em Journal of Machine Learning Research}, 17(1):7508--7538, 2016.

\bibitem{richards2021asymptotics}
Dominic Richards, Jaouad Mourtada, and Lorenzo Rosasco.
\newblock Asymptotics of ridge (less) regression under general source
  condition.
\newblock In {\em International Conference on Artificial Intelligence and
  Statistics}, 2021.

\bibitem{rubio2011spectral}
Francisco Rubio and Xavier Mestre.
\newblock Spectral convergence for a general class of random matrices.
\newblock {\em Statistics \& Probability Letters}, 81(5):592--602, 2011.

\bibitem{hansonwright}
Mark Rudelson and Roman Vershynin.
\newblock {Hanson-Wright inequality and sub-gaussian concentration}.
\newblock {\em Electronic Communications in Probability}, 18:1--9, 2013.

\bibitem{silverstein1995strong}
Jack~W. Silverstein.
\newblock Strong convergence of the empirical distribution of eigenvalues of
  large dimensional random matrices.
\newblock {\em Journal of Multivariate Analysis}, 55(2):331--339, 1995.

\bibitem{silverstein1995empirical}
Jack~W. Silverstein and Zhidong Bai.
\newblock On the empirical distribution of eigenvalues of a class of large
  dimensional random matrices.
\newblock {\em Journal of Multivariate analysis}, 54(2):175--192, 1995.

\bibitem{wu2020optimal}
Denny Wu and Ji~Xu.
\newblock On the optimal weighted $\ell_2 $ regularization in overparameterized
  linear regression.
\newblock {\em Advances in Neural Information Processing Systems}, 33, 2020.

\bibitem{xu2019number}
Ji~Xu and Daniel~J. Hsu.
\newblock On the number of variables to use in principal component regression.
\newblock {\em Advances in Neural Information Processing Systems}, 32, 2019.

\bibitem{zhang2005learning}
Tong Zhang.
\newblock Learning bounds for kernel regression using effective data
  dimensionality.
\newblock {\em Neural Computation}, 17(9):2077--2098, 2005.

\end{thebibliography}

 \end{document}